\documentclass[accepted]{uai2026} 

\usepackage[american]{babel}

\usepackage{amsmath,amsthm,amssymb}

\usepackage{multirow}

\usepackage{natbib} 

\usepackage{graphics,caption,subcaption}

\usepackage{algorithm,comment,algorithmicx,algpseudocode}
\algrenewcommand\algorithmicrequire{\textbf{Input:}}
\algrenewcommand\algorithmicensure{\textbf{Output:}}
\newcommand{\Input}{\Require}

\usepackage{mathtools,comment,setspace,color}
\theoremstyle{plain}
\newtheorem{theorem}{Theorem}[section]
\newtheorem{proposition}[theorem]{Proposition}
\newtheorem{lemma}[theorem]{Lemma}

\newtheorem{corollary}[theorem]{Corollary}
\newtheorem{remark}[theorem]{Remark}

\theoremstyle{plain} 
\newtheorem*{remark*}{Remark}
\newtheorem*{assumption*}{Assumption}
\newtheorem*{definition*}{Definition}
\newtheorem*{theorem*}{Theorem}
\newtheorem*{lemma*}{Lemma}
\newtheorem*{corollary*}{Corollary}

\usepackage{pifont}

\newcommand{\E}{\mathbb{E}}

\newcommand{\X}{\mathcal{X}}
\newcommand{\Y}{\mathcal{Y}}

\newcommand{\I}{\mathbb{I}}

\newcommand{\V}{\mathcal{V}}

\newcommand{\la}{\lambda}
\newcommand{\up}{\upsilon}
\newcommand{\al}{\alpha}
\newcommand{\be}{\beta}
\newcommand{\ga}{\gamma}

\usepackage{booktabs}

\usepackage{hyperref}

\usepackage[toc,page,header]{appendix}
\usepackage{minitoc}

\usepackage{tikz} 

\title{Near-Exponential Convergence Rates 
for kNN Classification\\ based on Boltzmann 
Margin 
}

\author[1]{Luyuan~Yang}
\author[1]{Shayan~Shafaei}
\author[1]{\href{mailto:<clan@ou.edu>?Subject=Your UAI 2026 paper}{Chao~Lan}}
\affil[1]{%
    School of Computer Science\\ 
    University of Oklahoma\\
    Norman, Oklahoma, USA
}

\begin{document}
\maketitle

\begingroup
\renewcommand\thefootnote{}
\footnotetext{This work has been 
accepted at UAI 2026.}
\addtocounter{footnote}{-1}
\endgroup

\begin{abstract}
Convergence-rate analysis for classifiers is 
often conducted under either Tsybakov margin 
or Massart margin. 
The former is a relatively weak condition 
that typically yields polynomial rates, 
while the latter is substantially stronger 
but can guarantee exponential rates. 
In this paper, we introduce 
a new condition, called \textit{Boltzmann margin},  
that bridges the gap between these two 
regimes. It is weaker than Massart 
margin, generally stronger than Tsybakov 
margin, and can imply many of their 
properties under suitable conditions. 
We apply Boltzmann margin to the analysis 
of kNN classifiers and establish the first \textit{near-exponential} convergence rates 
for kNN classification. 
We also present extensions of the main results 
and provide numerical evidence supporting 
the main theoretical implications.
\end{abstract}

\addtocontents{toc}{\protect\setcounter{tocdepth}{0}} 

\section{Introduction}

In machine learning, a fundamental question 
is how fast the classification error converges 
to the Bayes error as training data 
size increases. The convergence rate has 
been widely studied under two margin 
conditions: \textit{Tsybakov margin} 
\citep{tsybakov2004optimal} and  
\textit{Massart margin} \citep{massart2006risk}. 
Tsybakov margin is a relatively weak condition 
that assumes data decay polynomially 
fast towards the Bayes decision boundary 
and typically yields polynomial rates \citep{chaudhuri2014rates,doring2018rate}.  
Massart margin is substantially stronger,  
as it assumes no data exist near the boundary, 
but can guarantee exponential rates 
\citep{cabannnes2023case,vigogna2022multiclass}. 

There remains a substantial gap between 
the two margins. To establish any super-polynomial 
convergence rate, existing analyses typically 
require Massart margin, which may be 
unnecessarily strong for many problems. 
{\it Can there be a more balanced margin 
condition that yields super-polynomial 
convergence rates without assuming an 
absence of data near the boundary?} 
This question motivates our study. 

In this paper, we introduce \textit{Boltzmann margin}, 
a more balanced condition that assumes 
data  decay exponentially fast towards the 
Bayes decision boundary. It is weaker than Massart 
margin, generally stronger than Tsybakov margin, 
and can imply many of their properties under 
suitable conditions. 
Figure \ref{fig:margin} illustrates 
the relations among the three margins. 
In general, the decay behavior under 
Boltzmann margin lies between those under Massart 
and Tsybakov margins. Moreover, through suitable  
choices of its parameters, Boltzmann margin 
can resemble the decay behavior of 
either margin in certain regions, making 
it a flexible framework for convergence 
rate analysis. 

\begin{figure}[t!]
    \centering
    \includegraphics[width=0.9\linewidth]{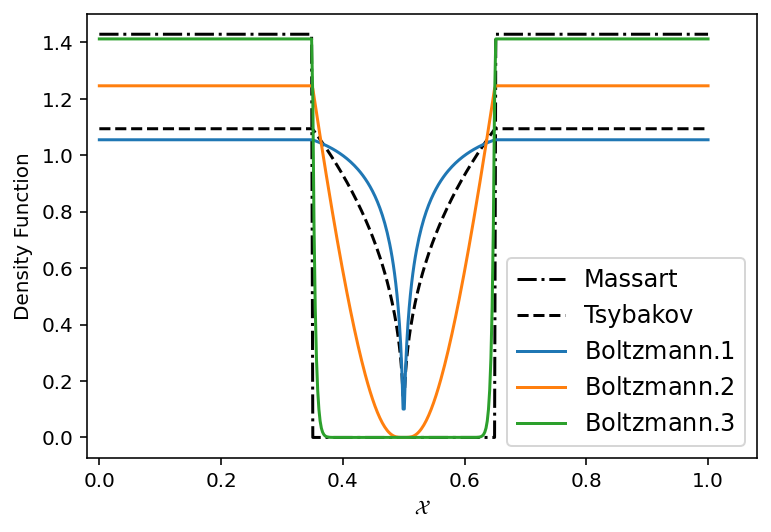}
    \vspace{-5pt}
    \caption{Example data densities 
    on $[0, 1]$ that satisfy different 
    margins respectively. 
    Bayes decision boundary is 0.5.}
    \label{fig:margin}
\end{figure}

\begin{table*}[t!]
\centering
\def\arraystretch{1}
\begin{tabular}{c|l|c|c|c|c}\toprule
\bf Model & 
\bf Work & \bf Posterior & \bf Margin 
& \bf Density 
& \bf Error Order \\ \midrule  
\multirow{10}{*}{kNN} 
& \citep{kulkarni2002rates} & 
H\"older & \text{\sffamily x} 
& \text{\sffamily x} & $n^{-\frac{2\up}{D}}$\\ 
& \citep{gyorfi2006rate} & H\"older* 
& \text{\sffamily x} & \text{\sffamily x} 
& $n^{-\frac{\up}{2\up+D}}$ \\ 
& \citep{kohler2007rate} 
& H\"older & Tsybakov* & strong 
& $n^{-\frac{\up(1+\be)}{2\up+D}}$ \\ 
& \citep{samworth2012supplement} &  
H\"older & Tsybakov & 
min.mass & $n^{-\frac{2(1+\be)}{4+D}}$ \\
& \citep{chaudhuri2014rates} & H\"older* 
& Tsybakov & \text{\sffamily x}
& $n^{-\frac{\up(1+\be)}{2\up+D}}$ \\
& \citep{chaudhuri2014rates} & H\"older* 
& Massart & \text{\sffamily x}
& $e^{- c' n}$ \\
& \citep{gadat2016as} & Lipschitz  
& Tsybakov & min.mass
& $n^{-\frac{1+\be}{2+D}}$ \\
& \citep{doring2018rate} &  
Lipschitz* & Tsybakov & 
\text{\sffamily x} & $n^{-\frac{1+\be}{2+D}}$ \\
& \citep{gyorfi2021universal} &  
H\"older & Tsybakov & 
\text{\sffamily x} 
& $n^{-\frac{\up(1+\be)}{2\up+D}}$ \\
& \citep{yang2025fast} 
&  H\"older & Tsybakov & 
\text{\sffamily x} 
& $n^{-\frac{3}{4}}$ \\
& Ours & H\"older & Boltzmann & 
\text{\sffamily x} 
& $e^{- c n^{\frac{\be}{\be+C_\up}}}$ \\[.1em] 
& Ours & H\"older & Massart & 
\text{\sffamily x} & $e^{- c n}$ \\\midrule
\multirow{5}{*}{ekNN} & 
\citep{biau2010rate} & Lipschitz & \text{\sffamily x} 
& \text{\sffamily x} & $n^{-\frac{2}{2+D}}$ \\ 
& \citep{samworth2012optimal} &  
continuous & Tsybakov & min.mass 
& $n^{-\frac{4}{4+D}}$ \\
& \citep{xue2018achieving} &  H\"older & Tsybakov & 
min.mass & $n^{-\frac{\up(1+\be)}{2\up+d}}$ \\
& \citep{qiao2019rates} &  H\"older* & Tsybakov & 
\text{\sffamily x} & 
$n^{-\frac{\up(1+\be)}{2\up+D}}$ \\
& Ours & H\"older & Boltzmann & 
\text{\sffamily x} 
& $e^{- c n^{\frac{g \be}{\be+C_\up}}}$ \\ \bottomrule
\end{tabular}
\caption{Convergence Rates for kNN Classifiers. 
Notations are loosely defined to cover both 
basic and generalized (marked by `*') 
definitions. $n$ is training data size;   
$\be$ is the main margin parameter;  
$\up$ is the main H\"older parameter 
(for H\"older smooth posterior 
function $\eta$); 
$D$/$d$ is input/intrinsic data 
dimension; $C_\up = 2 + d/\up$   
and $c',c,g$ are positive constants. }
\label{tab:summary_rate}
\end{table*}

\vspace{-5pt}

We apply Boltzmann margin to analyze 
k-nearest-neighbor (kNN) classifiers. 
kNN is popular due to its flexibility 
and interpretability  
\citep{chen2018explaining}, and 
is increasingly used to enhance 
the explainability of complex models 
\citep{khandelwalgeneralization,rajani2020explaining,dai2021towards,li2024nearest}. 
These developments motivate a deeper 
understanding of the fundamental properties of kNN. 

Convergence rates for kNN and ensemble 
kNN (ekNN) classifiers have long been 
studied. Table \ref{tab:summary_rate} 
summarizes the known results. 
As shown in the table, almost all studies  
assume a Tsybakov margin condition 
and establish polynomial rates. 
It is also shown in 
\citep[Appendix C.4]{chaudhuri2014rates}
that Massart margin can imply
an exponential convergence rate 
for kNN. 
In this paper, we apply Boltzmann margin 
to the analysis of kNN classifiers and establish 
their first \textit{near-exponential} 
convergence rates. 

In the remainder of the paper, we review related 
work in Section \ref{sec:related}, introduce 
and analyze Boltzmann margin in Section \ref{sec:preliminary}, present new results for 
kNN and ekNN in Sections \ref{sec:knn} and \ref{sec:eknn}, and provide numerical 
results in Section \ref{sec:experiment}.

\section{Related Work}
\label{sec:related}

To facilitate comparison of existing results, we adopt the notation used in Table \ref{tab:summary_rate} in this section. Formal notation for the present study is introduced in Section \ref{sec:preliminary}.

\subsection{kNN Classification}
\label{sec:related_knn}

Early studies on kNN convergence rates 
are based on relatively mild assumptions, 
notably H\"older smoothness of the posterior 
function $\eta$. 
\citep[Corollary 1]{kulkarni2002rates} 
presents a $O(n^{-\frac{2\up}{D}})$ convergence 
rate of 1NN error (to twice the Bayes 
error) when both $\eta$ and the 
conditional label variance are H\"older.
\citep{gyorfi2006rate} presents an 
asymptotic rate of $O(n^{-\frac{\up}{2\up+D}})$ 
for kNN when $\eta$ is generalized 
H\"older (which, informally, requires 
the average posterior within any open 
ball to approach the posterior at the 
center of the ball as the ball radius 
shrinks, with the convergence rate 
controlled by a H\"older-type parameter $\up$) 
and $k = n^{\frac{2\up}{2\up+D}}$. 

Later studies establish faster rates for kNN 
under Tsybakov margin and additional 
density assumptions. 
\citep[Theorem 6]{kohler2007rate} presents 
a $\tilde{O}(n^{-\frac{\up(1+\be)}{2\up+D}})$ 
rate when $\eta$ is H\"older, 
the data domain is bounded, the data density 
is bounded away from zero (i.e., strong 
density), and 
$k = \tilde{O}(n^{\frac{2\up}{2\up + D}})$. 
\citep[Theorem 1]{samworth2012supplement} 
presents a $O(n^{-\frac{2(1+\be)}{4+D}})$ rate 
when $\eta$ is H\"older, the data density 
has minimal mass (which, informally, requires 
any open ball centered at any point 
to have a mass that scales with 
its radius), 
and $k = O(n^{\frac{4}{4+D}})$. 
\citep[Theorem 3.1]{gadat2016as} presents 
a $O(n^{-\frac{1+\be}{2+D}})$ 
rate when $\eta$ is Lipschitz continuous, 
the data density has minimal mass, 
and $k = O(n^{\frac{2}{2+D}})$. 

Recent studies remain rooted in the Tsybakov 
margin and focus on relaxing conditions. 
\citep[Theorem 4]{chaudhuri2014rates} presents 
a $O(n^{-\frac{\up(1+\be)}{2 \up + D}})$ 
rate when $\eta$ is generalized 
H\"older and $k = O(n^{\frac{2\up}{2\up + D}})$. 
\citep[Theorem 1]{doring2018rate} presents 
a $O(n^{-\frac{1+\be}{2+D}})$ rate 
when $\eta$ is modified Lipschitz (which, 
informally, requires the posterior gap 
between two points to decrease as the 
radius of an open ball that contains them 
decreases) and $k = O(n^{\frac{2}{2+D}})$.  
\citep[Theorem 6]{gyorfi2021universal} presents 
a $O(n^{-\frac{\up(1+\be)}{2\up+D}})$ rate for 
multi-class kNN in a metric space when 
$\eta$ is (weighted) H\"older, the 
distribution function is continuous, 
and $k = O(n^{\frac{2\up}{2\up+D}})$. 
Recently, \citep[Theorem 4]{yang2025fast} 
shows that the convergence rate can 
vary across different phases   
and reach $O(D n^{-3/4})$ when $n \leq D^2$, 
provided that $\eta$ is H\"older, 
$D > 10$, and $k = O(n^{3/4})$. 
An exception appears in the 
appendix of \citep{chaudhuri2014rates}, 
where the established polynomial rate 
is extended to an asymptotic exponential 
rate under Massart margin. 

Convergence rates have also been 
studied for several kNN variants. 
Under Tsybakov margin, 
\citep[Theorem 4.3]{reeve2019fast} presents 
an asymptotic rate of $\tilde{O}(
n^{-\frac{\up(\be+1)}{2\up + 1}})$
for robustified kNN when $\eta$ is 
modified Lipschitz and $k = \tilde{O}( 
n^{\frac{2\up}{2\up + 1}})$;  
\citep[Theorem 5]{zhang2023improved} 
presents a $O(n^{-\frac{1+\be}{1+d}})$ 
rate for compressed kNN when $\eta$ 
is modified Lipschitz, the compressed 
data dimension is sufficiently large, 
and $k = O(n^{\frac{2}{2+d}})$. Without 
a margin assumption, 
\citep[Theorem 4.8]{rimanic2020convergence} 
presents an asymptotic rate of 
$O(n^{-\frac{1}{2+d}})$ 
for kNN trained in a 
bounded transformed domain  
when $\eta$ is probabilistic 
Lipschitz and $k = O(n^{\frac{2}{2+D}})$.

\subsection{Ensemble kNN Classification}
\label{sec:related_bageknn}

Ensemble learning is a popular strategy 
for improving classification performance 
by combining weak classifiers, typically learned 
from bootstrap samples, into a strong classifier 
\citep{dietterich2000ensemble,zhou2025ensemble}. 
Let $\tilde{n}$ be the bootstrap sample size 
and $m$ be the number of weak classifiers. 

Early studies on ensemble kNN assume $m = \infty$. 
\citep[Corollary 10]{biau2010rate} shows 
a bagged 1NN regression estimator has a 
$O(n^{-\frac{2}{2+D}})$ error rate  when 
$\eta$ is Lipschitz, the label variance is 
bounded, $D \geq 3$, $\tilde{n} \propto 
n^{\frac{D}{2+D}}$, and $k = O(n^{\frac{2}{2+D}})$. 
(The same rate applies to the induced bagged 
1NN classifier.)
\citep[Corollary 4]{samworth2012optimal} shows 
that under an implied Tsybakov margin, a  
bagged 1NN classifier has a $O(n^{-\frac{4}{4+D}})$ 
rate when $\eta$ is continuous, the data density 
has minimal mass, and $n^r \leq \tilde{n} 
\leq n^{1-r}$ for any $r \in (0, 1/2)$. 

Recent studies remain rooted in Tsybakov margin 
but do not require $m = \infty$.  
\citep{xue2018achieving} considers a special 
bagged 1NN classifier, where each bootstrap 
sample is relabeled using the training set 
and the kNN classification 
rule with $k = \tilde{O}(n^{\frac{2\up}{2\up + d}})$. 
It shows that the classifier has an asymptotic 
rate of $\tilde{O}(n^{-\frac{\up}{2\up+d}}) 
+ \tilde{O}(\tilde{n}^{-\frac{\up}{d}})$ 
when $\eta$ is H\"older and the data density 
has minimal mass, and that the first term 
dominates when 
$\tilde{n}/n = \Omega(n^{-\frac{\up}{2\up+d}})$. 
\citep{qiao2019rates} presents a 
$O(n^{-\frac{\up(1+\be)}{2\up+D}})$ rate
for bagged kNN when $\eta$ is modified 
H\"older and $k = O(\tilde{n}^{\frac{2\up}{2\up+D}} 
m^{-\frac{D}{2\up + D}}) \to \infty$ as 
$n \to \infty$.
In addition to standard ensemble kNN, 
\citep[Theorem 3]{hang2022under} 
designs an under-bagging kNN classifier 
for imbalanced classes; it shows that 
the classifier has an asymptotic rate of 
$\tilde{O}(n^{-\frac{\up(\be+1)}{2\up + D}})$ 
when $\eta$ is H\"older, 
$k = \tilde{O}(m \tilde{n}^{-\frac{D}{2\up+D}})$,  
and $m$ scales at a rate of 
$\tilde{O}(n^{\frac{2\up}{2\up + D}})$.

\subsection{Consistency for ekNN Classifier}

Bayes consistency is a desirable property 
of classifiers. Informally, a classifier is 
\textit{weakly consistent} if its expected 
error converges to the Bayes error, 
and \textit{strongly consistent} if 
its error converges to the Bayes error 
with probability one (over the random choice 
of training data). Strong consistency 
implies weak consistency 
\citep{devroye2013probabilistic}. 

Weak consistency is often easier to establish. 
In fact, most studies reviewed in Section \ref{sec:related_bageknn} already imply weak consistency, since their convergence rates 
imply that the expected classification error 
converges to the Bayes error. 
A few other studies also establish 
weak consistency for ekNN classifiers. 
\citep{hall2005properties} studies 
a bag of infinitely many 1NN classifiers, 
each trained on $\tilde{n}_+$ positive 
points sampled from $n_+$ candidates 
and $\tilde{n}_-$ negative points sampled 
from $n_-$ candidates, and shows that the 
classifier is weakly 
consistent if $\tilde{n}_+, \tilde{n}_-$ 
diverge and $\frac{\tilde{n}_+}{n_+}, 
\frac{\tilde{n}_-}{n_-}, \frac{\tilde{n}_+}{n_+}-\frac{\tilde{n}_-}{n_-}$ all converge 
to zero. \citep{biau2010layered} 
shows that a bag of $m$ 1NN classifiers 
is weakly consistent if $m, \tilde{n}$ 
diverge and $\frac{\tilde{n}}{n} \to 0$. 
\citep{JMLR:v9:biau08a} further shows that 
a probabilistic version of this 
classifier, where $\tilde{n}$ depends 
on a sampling probability $q_n$ for 
each candidate point, is weakly consistent 
if $\E \tilde{n} = n q_n$ diverges 
and $q_n \to 0$.

Strong consistency often requires stronger 
conditions. kNN is shown to be strongly 
consistent if $k \to \infty$ and $k/n \to 0$ 
\citep[Chapter 11]{devroye2013probabilistic}. 
1NN is not strongly consistent 
\citep{cover1967nearest}, but its  
variants based on compressed 
training data are strongly consistent 
\citep{kontorovich2015bayes,hanneke2020ita}. 
To the best of our knowledge, strong consistency 
for ekNN has not been established. 
This paper presents the first such 
guarantee.

\section{Preliminaries}
\label{sec:preliminary}

In this section, we introduce 
notation and concepts. 
When clarity is needed, a probability 
(or expectation) is denoted by $\Pr_Z$ 
(or $\E_Z$) when taken with respect 
to the randomness of $Z$. 
As in prior studies, we focus on 
binary classification. 

Let $\X$ be a data space equipped with 
norm $||\cdot||$ and $\Y = \{0, 1\}$ be a 
label set. Let $B(x, r) = \{x' \in \X; ||x - x'|| 
\leq r\}$ denote a closed ball centered 
at $x$ with radius $r$. For $\X$, 
let $D$ be its input dimension   
and $d$ be its \textit{intrinsic dimension}  
introduced in \citep{xue2018achieving},  
i.e., an integer for which there exists 
a constant $C_d > 0$ such that $\Pr_{z}\{z 
\in B(x, r)\} \geq \min(C_d r^d, 1)$  
for all $x \in \X$ and $r > 0$. 
While most known rates depend on $D$, 
our rates depend on $d$.
 
Let $X \in \X$  be a random data point 
and $Y \in \Y$ be its label. The posterior 
function $\eta$ is defined by
\begin{equation*}
\eta(x) = \E_{Y} [Y \mid X = x].   
\end{equation*}
The following assumption is made 
throughout the analysis. 
\begin{assumption*}[H\"older Smooth] 
The posterior function $\eta$ 
is $(\la, \up)$-H\"older smooth 
for constants $\la, \up >0$, 
i.e., for any $x, x' \in \X$, 
we have $|\eta(x) - \eta(x')| 
\leq \la \cdot ||x - x'||^{\up}$.
\end{assumption*}
Let $\hat{y}$ be a classifier 
and $er(\hat{y}) = {\Pr}_{X,Y} 
\{\hat{y}(X) \neq Y\}$ be 
its error. Let $\hat{y}_*$ be a 
\textit{Bayes classifier}, i.e., 
for all $x \in \X$: 
\begin{equation*}
\hat{y}_*(x) = \I \{\eta(x) \geq 1/2\},  
\end{equation*} 
where $\I$ is the indicator function. 

For a classifier $\hat{y}$ trained on 
a random sample $Z \subseteq \X \times \Y$, 
its \textit{expected excess error} is defined as 
\begin{equation}
\label{eq:deferror}
||\hat{y} - \hat{y}_*||_e := 
\E_Z [er(\hat{y})] - er(\hat{y}_*).   
\end{equation}
Unless otherwise specified, all convergence 
rates discussed in this paper are w.r.t. 
the expected excess error in (\ref{eq:deferror}).

\subsection{Boltzmann Margin}

We first revisit two popular margin conditions. 
Let $\al$, $\be$, $\ga$, $t_*$ be positive constants.  
A distribution on $\X \times \Y$ is said to 
have a $t_*$ Massart margin 
\citep{massart2006risk} if
\begin{equation*}
{\Pr}_X \{ |1/2 - \eta(X)| \leq t_*\} = 0,     
\end{equation*}
and a $(\al,\be)$ \textit{Tsybakov margin} 
\citep{tsybakov2004optimal} if 
\begin{equation*}
\forall t>0,\ \ {\Pr}_X \{ |1/2 
- \eta(X)| \leq t \} \leq \al t^\be.  
\end{equation*}
It is clear that Massart margin is 
much stronger than Tsybakov margin. 
In the following, we introduce a new margin 
condition that interpolates between them. 
It is named after the Boltzmann factor 
in physics \citep{lei2020improved} 
due to the similarity of their forms. 

\begin{definition*}[Boltzmann Margin] 
A distribution on $\X \times \Y$ is said 
to have a $(\al,\be, \ga)$ Boltzmann 
margin if 
\begin{equation}
\label{eq:boltzmann}
\forall t>0,\ {\Pr}_X \{ |1/2 - \eta(X)| \leq 
t \} \leq \al \exp(-\ga t^{-\be}). 
\end{equation}
\end{definition*}

Figure \ref{fig:margin} illustrates 
the relations among the margins. 
The three Boltzmann curves are generated 
with $\be = 0.1, 0.4, 0.8$, respectively, 
while fixing $\al, \ga$. As $\be$ increases, 
the Boltzmann curve approaches the 
Massart curve. As $\be$ decreases, 
it approaches the Tsybakov curve and 
even becomes weaker in certain regions.  
These relations illustrate the flexibility 
of the Boltzmann margin in convergence analysis. 

We now formalize the relations among 
the margins. The first lemma shows 
that Massart margin implies Boltzmann 
margin, and the reverse holds under 
suitable parameter choices. 

\begin{lemma}
\label{lem:relation_massart}
A $t_*$ Massart margin implies a 
$(\al, \be, \ga)$ Boltzmann 
margin for $\al, \be, \ga > 0$ 
satisfying $\al \exp(-\frac{\ga}{t_*^\be}) = 1$. 
Reversely, for any $t_* > 0$, 
a $(1, \be, t_*^{\be})$ Boltzmann margin 
approaches the $t_*$ Massart margin 
condition as $\be \to \infty$, provided 
that $|1/2 - \eta(X)| \neq t_*$ 
almost everywhere.
\end{lemma}

The next lemma shows Tsybakov margin 
implies Boltzmann margin and the reverse 
holds far from the decision boundary. 

\begin{lemma}
\label{lem:relation_tsybakov}
A $(\al, \be, \ga)$ Boltzmann margin 
implies a $(\frac{\ga}{\al}, \be)$ Tsybakov margin. 
Reversely, a $(\al, \be)$ Tsybakov margin 
satisfies the $(\al', \be',\ga')$ 
Boltzmann margin property (\ref{eq:boltzmann}) 
for $t \in [(\frac{\ga'}{\ln (\al'/\al)})^{1/\be'}, 
0.5]$, if $\al', \be' > 0$ and $\ga' \leq \ln (\al'/\al)$. 
\end{lemma}

Figure \ref{fig:margin_blzvstsy} 
illustrates the reversed relation. 
The decay is much slower under Boltzmann 
margin except near the decision boundary. 
Although this may seem less significant 
from the margin perspective, where the 
focus is often on near-boundary properties, 
it is highly relevant to machine learning. 
The reason is that far-boundary properties 
often govern non-asymptotic model behavior, 
which is particularly important in practice 
because training data are finite.

\begin{figure}[t!]
    \centering
    \includegraphics[width=0.85\linewidth]{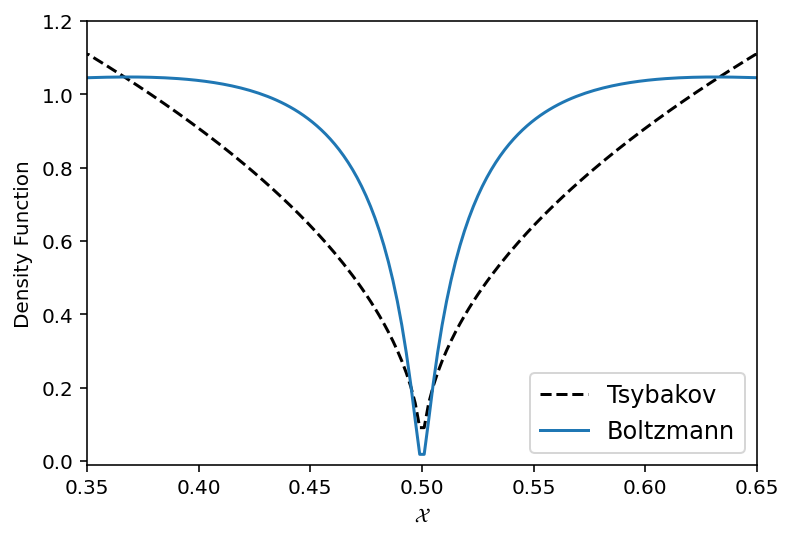}
    \vspace{-5pt}
    \caption{Example Data Densities}
    \label{fig:margin_blzvstsy}
\end{figure}

To better understand the above implication, 
we first establish a monotonicity property  
of the Boltzmann margin.

\begin{lemma} 
\label{lem:relation3}
A $(\al, \be, \ga)$ Boltzmann margin implies 
a  $(\al', \be', \ga')$ Boltzmann margin whenever 
(i) $\al' \geq \al$, (ii) $\be' \geq \be$ or (iii) $\ga' \leq \ga$, with the remaining parameters fixed. 
\end{lemma}

Now consider a model trained on $n$ data 
points. Suppose data distribution has a 
$(\al, \be)$ Tsybakov margin, and set 
$t = 1/n$. 
Then Lemma \ref{lem:relation_tsybakov} 
implies the distribution also satisfies 
the $(\al', 1, \ga')$ Boltzmann margin 
property whenever 
\begin{equation*}
2 \leq n \leq \ga'^{-1} \ln(\al'/\al).
\end{equation*}
By Lemma \ref{lem:relation3}, one may pick 
$\ga' \leq 0.001$ and $\alpha' \geq \alpha^2$, 
yielding the range $n \in [2,2000]$. 
This suggests the model may exhibit 
Boltzmann-margin behavior for moderate $n$.

\begin{remark}
\label{rem:nobeta}
It is worth noting that in Lemma 
\ref{lem:relation_tsybakov}, for 
a $(\al, \be)$ Tsybakov margin to imply 
a Boltzmann margin, the constraint on $t$ 
does not depend on $\be$. This is because 
we set $t \leq 0.5$ so that 
$t^\be \leq 1$ for any $\be$. 
This setting is often sufficient; 
for example, when $t=1/n$, 
it covers all $n \geq 2$.       
\end{remark}

\section{Results for kNN Classifier}
\label{sec:knn}

Let $S$ be an i.i.d. sample of $n$ points
from $\X \times \Y$. The kNN regression 
estimator for $\eta(x)$ based on $S$ is 
\begin{equation*}
\hat{\eta}(x; S) = 
\frac{1}{k} \sum_{(x', y') \in S} 
\I\{x' \in N_k(x; S) \wedge y' = 1\}, 
\end{equation*}
where $N_k(x; S)$ is the set of $k$ 
nearest neighbors of $x$ in $S$. 
The induced kNN classifier is 
defined by
\begin{equation*}
\hat{y}(x) = \I \{\hat{\eta}(x; S) \geq 1/2\}.    
\end{equation*}
As in prior studies, we assume ties occur
with probability zero e.g., ties in the 
distance to the $k$th nearest neighbor 
or in the decision rule $\eta(x)=1/2$.

Our main result is stated as follows. 

\begin{theorem}
\label{thm:main_knn_rate}
Suppose data distribution has a 
$(\al, \be, \ga)$ Boltzmann margin. 
Let $C_\up = 2 + d/\up$. 
If $k = O(n^{\frac{\be+2}{\be+C_\up}})$, 
then for $n \geq (C' \ln n)^{1+C_\up/\be}$: 
\begin{equation*}
||\hat{y} - \hat{y}_*||_e
\leq C_\al \exp\left(- C_\ga
 n^{\frac{\be}{\be + C_\up}} \right),     
\end{equation*}    
where $C_\al = 2 + \al$, $C_\ga = \min(\ga, 1/C')$ 
and $C'>0$ is a constant depending on $\la, \up$ 
and the structure of $\X$. 
\end{theorem}

Theorem \ref{thm:main_knn_rate} establishes 
the first near-exponential convergence rate 
for kNN classification, and the exponent 
depends on $\be, d$ and $\up$. 
In the following, we discuss its implications. 

Larger $\be$ makes the rate closer to 
exponential. This makes sense because, 
by Lemma \ref{lem:relation_massart}, 
increasing $\be$ causes Boltzmann 
margin to approach Massart margin,  
and the latter is known to warrant 
exponential rates \citep{cabannnes2023case,vigogna2022multiclass}. 
(However, this does not trivially imply that kNN 
has an exponential rate under Massart margin. 
We establish such a result later.) 

Smaller $d/\up$ makes the rate closer 
to exponential. This is consistent with prior 
results, which show smaller $D/\up$ implies 
faster rates (Table \ref{tab:summary_rate}). 
A minor difference is that our rate depends 
on the intrinsic dimension $d$ instead of $D$. 

The theorem also yields several
other observations. 
First, larger $\up\be/d$ permits a slower 
growth rate for $k$ but also leads to 
slower convergence -- a similar tradeoff 
appears in prior kNN convergence guarantees. 
Second, the assumed scaling rate for 
$k$ is slightly faster than in prior 
results, as we assume $k = O(n^{\frac{2+\be}{2 
+ \be + d/\up}})$ while they 
assume $k = O(n^{\frac{2}{2+d/\up}})$ or 
$O(n^{\frac{2}{2+D/\up}})$. 
Third, our rate holds for large $n$, 
which is a common constraint in 
existing rates that are either asymptotic 
\citep{gyorfi2006rate,reeve2019fast,rimanic2020convergence,xue2018achieving} or multi-phase \cite{yang2025fast}. 
The next remark shows that, under 
suitable conditions, the large-$n$ 
condition in Theorem \ref{thm:main_knn_rate} 
can be removed without affecting the 
convergence rate. 

\begin{remark}
\label{rem:thm_knn_rate}
Theorem \ref{thm:main_knn_rate} implies 
that, for $n \geq 1$: 

(i) If $k = O(n^{\frac{\be+2}{\be+C_\up}})$, 
then there exists a constant $\tilde{C}_\al > 0$ such 
that $||\hat{y} - \hat{y}_*||_e \leq \tilde{C}_{\al} 
\exp\left(-  C_\ga n^{\frac{\be}{\be + C_\up}} \right)$. 

(ii) Take $\be \to \infty$. Then 
$||\hat{y} - \hat{y}_*||_e 
\leq C_\al' \exp\left(- C_\ga n \right)$ 
for $k = O(n)$,  
where $C_\al' = \lim_{\be \to \infty} \tilde{C}_\al 
< \infty$.  
\end{remark}

Claim (ii) suggests that Boltzmann margin may 
warrant an exponential rate under sufficiently 
fast data decay. 
This is not a trivial consequence of (i), 
since one must verify that $\tilde{C}_\al$ 
remains bounded as $\be \to \infty$. 
Combined with Lemma \ref{lem:relation_massart}, 
this implies an exponential and non-asymptotic 
convergence rate for kNN classification  
under Massart margin, which is formalized 
in the next subsection. 
Hereafter, unless otherwise 
specified, the constants $C_\ga$, $C'$, $C_\up$
are as defined in Theorem \ref{thm:main_knn_rate}, 
while $\tilde{C}_\al$, $C_\al'$ are as 
defined in Remark \ref{rem:thm_knn_rate}. 

\subsection{Extended Results}

In this section, we present three 
corollaries of Theorem \ref{thm:main_knn_rate}. 
The first one provides a 
probabilistic convergence rate. 

\begin{corollary}[Probabilistic]
\label{cor:main_knn_rate_v2}
Suppose data distribution has a 
$(\al, \be, \ga)$ Boltzmann margin. 
If $k = O(n^{\frac{\be+2}{\be+C_\up}})$, 
then for any $\delta > 0$ and $n \geq 1$: 
\begin{equation*}
er(\hat{y}) - er(\hat{y}_*) \leq 
\tilde{C}_\al \delta^{-1} \exp\left(- 
C_\ga n^{\frac{\be}{\be + C_\up}} \right),  
\end{equation*}
with probability at least $1 - \delta$ 
over the random choice of $S$.
\end{corollary}

This probabilistic rate is near-exponential, 
in contrast to the polynomial probabilistic rates 
in \citep{reeve2019fast,rimanic2020convergence}. 
Moreover, the dependence on $\delta$ is 
multiplicative in our bound, whereas it is 
additive in the prior bounds, e.g., 
$\tilde{O}(n^{-\frac{\up(1+\be)}{2\up+1}} 
+ \delta)$. With the common choice 
$\delta = O(1/n)$, our bound remains 
near-exponential. 

Theorem \ref{thm:main_knn_rate} also 
yields convergence rates under other margins. 
The first is under Tsybakov margin. 

\begin{corollary}[Tsybakov] 
\label{cor:main_knn_rate_v3}
Suppose data distribution has 
a $(\al', \be')$ Tsybakov margin  
and let $\al, \be, \ga>0$ be 
constants satisfying 
$\ga \leq \ln(\al/\al')$. 
If $k = O(n^{\frac{2 +\be}{C_\up + \be}})$, 
then 
\begin{equation*}
||\hat{y} - \hat{y}_*||_e
\leq \tilde{C}_{\al} \exp(- C_{\ga} 
n^{\frac{\be}{\be + C_\up}}),   
\end{equation*}
for $2 \leq n \leq \left( \frac{1}{\ga} 
\ln(\al/\al') \right)^{1 + C_\up/\be}$.  
\end{corollary}

This can be viewed as a multi-phase 
convergence rate, 
where the rate is near-exponential 
within the above range and 
becomes polynomial beyond it. 
Multi-phase rates have recently been studied 
in \citep{yang2025fast}. The authors show
that the rate can be improved to $O(n^{-3/4})$ 
when $n \leq D^2$, but all rates remain 
polynomial. By contrast, we show 
the rate can be near-exponential 
over a finite range of $n$, and 
this range can be $O(D^2)$ when 
$\al$ is large or $\ga$ is small. 

It is worth noting that $\be'$ 
does not appear in the error bound 
or the range of $n$, for the reason 
explained in Remark \ref{rem:nobeta}. 
Moreover, under suitable conditions, 
the bound may be expressed in terms 
of the Tsybakov-margin parameters and 
the admissible range only. For example,
when $\al$ is large, $\ga$ is small,  
and $\be \to \infty$, it follows 
that the range becomes nearly 
$n \leq \frac{1}{\ga} \ln(\al/\al') := B$ 
and the error bound becomes $\frac{n}{B} 
\exp(- \ln(\al/\al') + \ln (\al+2))$;
optimizing it over $\al$ yields a bound 
independent of the Boltzmann 
margin parameters. 

The last corollary specializes the 
rate under Massart margin. 
\begin{corollary}[Massart]
\label{cor:main_knn_rate_v4}
Let $k = O(n)$. If data distribution 
has a $t_*$ Massart margin, then 
\begin{equation}
\label{eq:massart_exprate_mp}
||\hat{y} - \hat{y}_*||_e \leq 
C \exp(- C_* n), 
\end{equation}    
for all $n \geq C_*^{-1} \ln n$,    
where $C$ and $C_* \propto t_*^{C_\up}$ 
are positive constants. 
Moreover, for every $t_* > 0$, 
there exists a data distribution 
satisfying the $t_*$ Massart margin 
for which (\ref{eq:massart_exprate_mp}) 
holds for all $n \geq 1$ (possibly 
with different constants).   
\end{corollary}

The corollary establishes an exponential 
convergence rate for kNN classification  
under Massart margin. 
We can compare it with the rate in 
\citep[Appendix C.4]{chaudhuri2014rates}. 
First, the two rates are developed using 
fundamentally different techniques (ours 
via Boltzmann margin). Second, both rates 
are $O(\exp(-Cn))$, where $C \propto t_*^p$ 
for some constant $p$. 
The prior rate has $p = 2 + \frac{1}{\al'}$, 
where $\al' = \up/D$ is the 
generalized H\"older parameter,  
while we have $p = C_{\up} 
= 2 + \frac{d}{\up} = 2 + \frac{d}{D} 
\frac{1}{\al'}$. 
Since $d/D < 1$, our $p$ is smaller, 
implying a larger $t_*^p$ for 
 $t_* \in (0,1)$. Consequently, the 
exponential constant in our bound 
increases slightly faster with $t_*$.
Finally, both rates require $k = O(n)$. 

The corollary contains two distinct 
claims: one holds under arbitrary Massart 
margins but only yields an asymptotic rate, 
while the other only holds under certain 
Massart margins but yields a non-asymptotic 
rate. Moreover, condition $k = O(n)$ 
is stronger than 
$O(n^{\frac{\be+2}{\be+C_\up}})$ 
used in near-exponential rates.

\section{Results for ekNN Classifier}
\label{sec:eknn}

In this section, we present new convergence 
rates for ensemble kNN (ekNN) classification 
under Boltzmann margin. 

For ensemble classifiers, the excess error 
definition (\ref{eq:deferror}) requires a 
slight adjustment to account for the additional 
randomness introduced by bootstrapping. 
There are two common settings. 
One assumes a training set is sampled 
from the population and bootstrap samples 
are drawn from it. In this case, one can 
make claims that account for the additional 
bootstrapping randomness, e.g., 
\citep{xue2018achieving} establishes error rates 
with high probability over the random choices 
of both the training set and bootstrap samples. 
The other assumes bootstrap 
samples are drawn directly from the 
population and treats their collection 
as the training set, e.g., \citep{biau2010layered}; 
in this case, (\ref{eq:deferror}) incorporates 
all randomness, so no special treatment is 
needed. For convenience, we adopt the second 
setting, but our results also apply 
to the first setting 
(Remark \ref{rem:eknnmain}).

Let $\tilde{S} := \{ S_1, \ldots, S_m \}$ 
be a collection of $m$ samples, where 
each $S_j$ contains $\tilde{n}$ points i.i.d. 
sampled from $\X \times \Y$, 
and all samples are drawn i.i.d. 
Hence $n = m \tilde{n}$. 
Consider the following bagged 
kNN regression estimator  
\citep{sutton2005classification,hastie2009elements,hang2022under} based on $\tilde{S}$:  
\begin{equation*}
\hat{\eta}_e(x; \tilde{S}) 
= \frac{1}{m} \sum_{j=1}^m \hat{\eta}(x; S_j).    
\end{equation*}
The induced ekNN classifier is 
\begin{equation*}
\hat{y}_e(x) = \I\{ \hat{\eta}_e
(x; \tilde{S}) \geq 1/2\}. 
\end{equation*}

Our main result is stated as follows. 

\begin{theorem}
\label{thm:main_eknn_rate}
Suppose data distribution has a 
$(\al, \be, \ga)$ Boltzmann margin. 
If $k = O(\tilde{n}^{\frac{\be + 2}{\be 
+ C_\up}})$, then 
\begin{equation}
\label{eq:thm_eknn_mp}
||\hat{y}_e - \hat{y}_*||_e
\leq C_{\al}
\exp(- C_\ga \tilde{n}^{\frac{\be}{\be + C_\up}}),     
\end{equation}  
for $\tilde{n} \geq 
\left( 2 \max \{ \ln \tilde{n} 
, 2 \ln m \}/C' \right)^{1+C_\up/\be}$. 

Moreover, if $m = O(\sqrt{\tilde{n}})$, 
then (\ref{eq:thm_eknn_mp}) holds for 
all $\tilde{n} \geq 1$, with $C_\al$ 
replaced by a possibly different 
constant $\tilde{C}_{\al} > 0$.
\end{theorem}

We have some interesting observations from 
the theorem. First, recall that the known 
convergence rates for ekNN classifiers 
are polynomial and typically assume 
$\tilde{n} = \Theta(n^g)$ 
for some constant $g$, e.g., $\frac{D}{2+D}$ or $\frac{\up + d}{2 \up + d}$ 
(Section \ref{sec:related_bageknn}).   
Under the same choice of $\tilde{n}$, 
our theorem yields
\begin{equation*}
\label{eq:eknn_Nrate}
||\hat{y}_e - \hat{y}_*||_e 
= O(e^{-C_\ga n^{\frac{g \be}{\be + C_\up}}}).  
\end{equation*}
This rate is slightly weaker than the one we 
established for kNN classification, but can be  
close when $g$ is near 1, e.g., under 
the choice $g = \frac{D}{2+D}$ with large $D$. 
As a result, it is generally faster than 
the known polynomial rates. 

Second, unlike \citep{biau2010rate,samworth2012optimal}, 
our rate does not require $m = \infty$. 
In fact, if $m = \infty$, our rate 
only holds for $\tilde{n} = \infty$, 
which limits its practical relevance. 
That said, Theorem \ref{thm:main_eknn_rate}  
does allow $m$ to diverge no 
faster than $\tilde{O}(e^{\tilde{n}})$.
A related condition appears in 
\citep{qiao2019rates}, which requires 
$m$ to diverge no faster than 
$O(\tilde{n}^{2 \V/D})$ to ensure
$k = O(\tilde{n}^{\frac{2\up}{2\up+D}} 
m^{-\frac{D}{2\up + D}}) \to \infty$. 
Our condition is weaker in two respects: 
(i) it allows but does not require the 
divergence of $m$; (ii) $e^{\tilde{n}}$ 
grows exponentially faster than 
$\tilde{n}^{2 \V/D}$, allowing $m$ to 
diverge substantially faster. 

Finally, Theorem \ref{thm:main_eknn_rate} 
also applies to the alternative bootstrap 
setting discussed earlier in this section. 
In the following remark, conditions (ii) 
and (iii) are standard in ekNN analysis, 
e.g., \citep{xue2018achieving,qiao2019rates}.

\begin{remark}
\label{rem:eknnmain}
Theorem \ref{thm:main_eknn_rate} 
also applies when 
(i) a training set $S$ is sampled 
i.i.d. from the population, (ii) 
subsets $S_1, \ldots, S_m$ 
are drawn i.i.d. from $S$, 
and (iii) in each subset, points 
are sampled without replacement.
\end{remark}

\subsection{Other Results}

Theorem \ref{thm:main_eknn_rate}  also 
implies a probabilistic convergence rate. 

\begin{corollary}[Probabilistic]
\label{cor:main_eknn_rate_v3}
Suppose data distribution has $(\al, \be, \ga)$ 
Boltzmann margin. If  
$k = O(\tilde{n}^{\frac{\be + 2}{\be + C_\up}})$ 
and $m = O(\sqrt{\tilde{n}})$, then for 
any $\delta > 0$ and $\tilde{n} \geq 1$, 
\begin{equation*}
er(\hat{y}_e) - er(\hat{y}_*) \leq 
\tilde{C}_\al \delta^{-1} \exp(- 
C_\ga \tilde{n}^{\frac{\be}{\be + C_\up}}),  
\end{equation*}
with probability at least $1 - \delta$ 
over the random choice of $\tilde{S}$, 
where $\tilde{C}_\al$ is as defined 
in Theorem \ref{thm:main_eknn_rate}. 
\end{corollary}

We can compare the above with the 
asymptotic rate for the bagged 1NN classifier 
in \citep{xue2018achieving}, which is 
$\tilde{O}\left(n^{-1}\ln (1/\delta)
\right)^{\frac{1}{2+d/\up}}$ if  
$\frac{\tilde{n}}{n} = 
\Omega(n^{-\frac{\up}{2\up+d}})$.  
Under the same choice of $\tilde{n}/n$, 
our bound becomes 
$O(e^{- C_\ga n^{\frac{1 + d}{C_\up} \cdot \frac{\be}{\be + C_\up}}})$. 
This rate is weaker than the asymptotic 
rate we established for kNN, 
but can be close when $d/\up$ is small or 
$\be$ is large -- in either case, our bound 
remains near-exponential and can be 
faster than the prior polynomial rate.

Our next result is a strong consistency 
guarantee for ekNN. Recall that a classifier 
$\hat{y}$ trained on a random sample $Z$ is 
strongly Bayes consistent      
if $\Pr_Z\{ \lim_{|Z| \to \infty} 
er(\hat{y}) = er(\hat{y}_*)\} = 1$ 
\citep{devroye2013probabilistic}. 

\begin{theorem}[Strong Consistency]
\label{thm:main_eknn_consistency}
Suppose data distribution has a 
$(\al, \be, \ga)$ Boltzmann margin. 
Then $\hat{y}_e$ is strongly 
Bayes consistent 
if $k = O(\tilde{n}^{\frac{\be + 2}{\be + C_\up}})$ 
and $m = \tilde{O}(e^{\tilde{n}})$. 
\end{theorem}

To the best of our knowledge, this is the 
first result that formally establishes strong 
consistency for ekNN classification.
We can compare its condition on $k$ 
with those for weak consistency, such as 
$O(n^{\frac{4}{4+D}})$ or 
$O(n^{\frac{\up(1+\be)}{2\up+D}})$ 
(Section \ref{sec:related_bageknn}). 
Our $O(\tilde{n}^{\frac{\be + 2}{\be + 2+ d/\up}})$ 
has similar implications: $k$ needs to 
scale with $\tilde{n}$, and a larger 
data dimension allows 
it to scale more slowly. 
A major difference is in the impact of $\be$: 
in prior conditions, an arbitrarily large $\be$ 
can require $k$ to scale arbitrarily fast 
alongside $n$, e.g., $k = O(n^5)$; in our condition, 
$k = O(\tilde{n})$ even when $\be \to \infty$. 

We should also clarify that the two 
big-$O$ expressions in Theorem \ref{thm:main_eknn_consistency} serve different purposes. The expression 
for $k$ specifies a choice that scales with $\tilde{n}$, whereas the expression for $m$ does not require 
$m$ to scale.

\section{Numerical Results}
\label{sec:experiment}

We first construct a distribution 
that satisfies Boltzmann margin  
and sample data from it for 
empirical evaluation. 

Consider $\X = [0, 1]$ 
with posterior $\eta(X) = X$. 
The Bayes decision 
boundary is at $X = 0.5$. 
Figure~\ref{fig:eta_error} 
illustrates the posterior  function 
and the corresponding Bayes 
error under a uniform 
distribution (i.e., the 
area of the highlighted region).

To evaluate the impact of Boltzmann
margin, we construct a family of
distributions whose density decays
near the Bayes decision boundary 
and remains constant elsewhere. 
Given constants $\al, \be, \ga > 0$, 
define the density function 
\begin{equation}
\label{eq:density_blz_main}
f_b(x) = \begin{cases}
\frac{\al \be \ga}{2N_T} \cdot 
\frac{e^{-\frac{\ga}{|x-0.5|^{\be}}}}{
|x-0.5|^{\be + 1}}, 
& \text{if  } |x - 0.5| \leq T, \\[.75em]  
\frac{\al \be \ga}{2N_T} \cdot 
\frac{e^{-\ga T^{-\be}}}{T^{\be + 1}}, 
& \text{if  } |x - 0.5| > T. 
\end{cases}
\end{equation}
It can be shown that the distribution 
induced by $f_b$ satisfies 
$(\al, \be, \ga)$ Boltzmann margin. 
The detailed construction and 
verification are presented in Appendix \ref{sec:app_numerical_density}. 

\begin{figure}[t!]
    \centering
    \includegraphics[width=.8\linewidth]{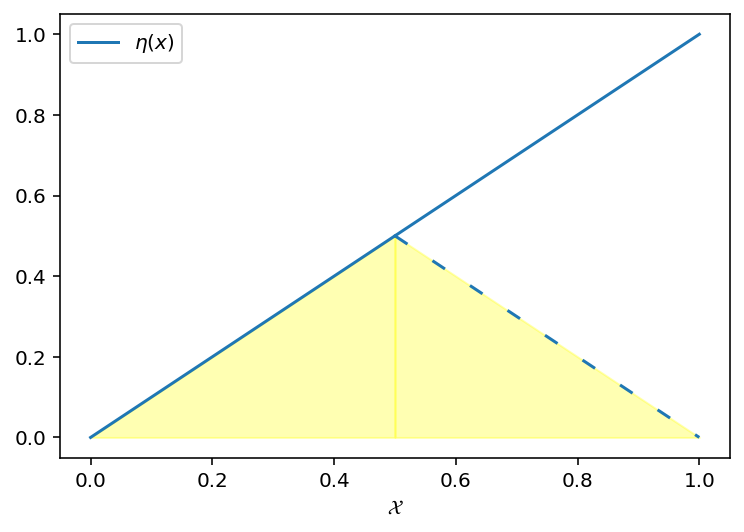}
    \vspace{-5pt}
    \caption{Posterior Function}
    \label{fig:eta_error}
\end{figure}

\begin{remark*}
In Figure \ref{fig:margin}, the 
Boltzmann-margin density is generated 
based on (\ref{eq:density_blz_main}) with 
$(T, \ga) = (0.15,10)$ and $\be = 0.1,0.2,0.8$ 
for Boltzmann curves 1, 2, 3 respectively.

Other densities are constructed in 
similar ways. For Massart margin, it is 
$f_t(x) = \left\{ \begin{matrix}
0, & |x - 0.5| \leq T\\ 
\frac{1}{1 - 2T}, & |x - 0.5| > T, 
\end{matrix}\right.$. 
For $(\al', \be')$ Tsybakov margin, 
it is 
\begin{equation*}
f_t(x) =\begin{cases}
\frac{\al' \be'}{2N'_T} 
\cdot |x - 0.5|^{\be' - 1}, 
& \text{if  } |x - 0.5| \leq T,
\\[.75em] 
\frac{\al' \be'}{2 N'_T} \cdot 
T^{\be' - 1}, 
& \text{if  } |x - 0.5| > T, 
\end{cases}     
\end{equation*}
where 
$N'_T = \al' T^{\be'} + 2 C'_T (0.5-T)$ 
and $C'_T = \frac{\al' \be'}{2}T^{\be' - 1}$.     
In the figure, $(\al', \be') = (1,1.5)$ 
and $T = 0.15$ for both. 
\end{remark*}

After constructing Boltzmann-margin distribution, 
we generate data points from it by applying 
the rejection sampling technique 
\cite{bishop2006pattern}. 
The detailed sampling procedure is described in 
Appendix \ref{sec:app_numerical_sample}.

\begin{figure}[t!]
    \centering
    \includegraphics[width=.86\linewidth]{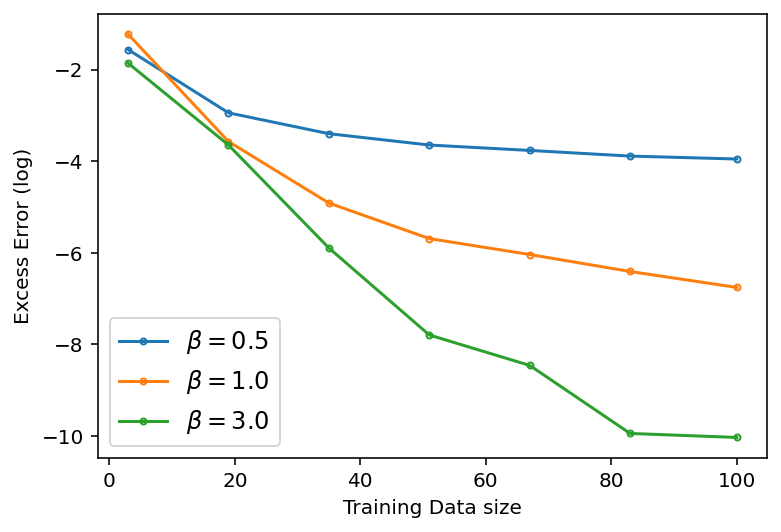}
    \vspace{-8pt}
    \caption{kNN Classification Performance}
    \label{fig:Boltzmann_knn}
\end{figure}

\vspace{-2pt}
\begin{figure}[t!]
    \centering
    \includegraphics[width=.86\linewidth]{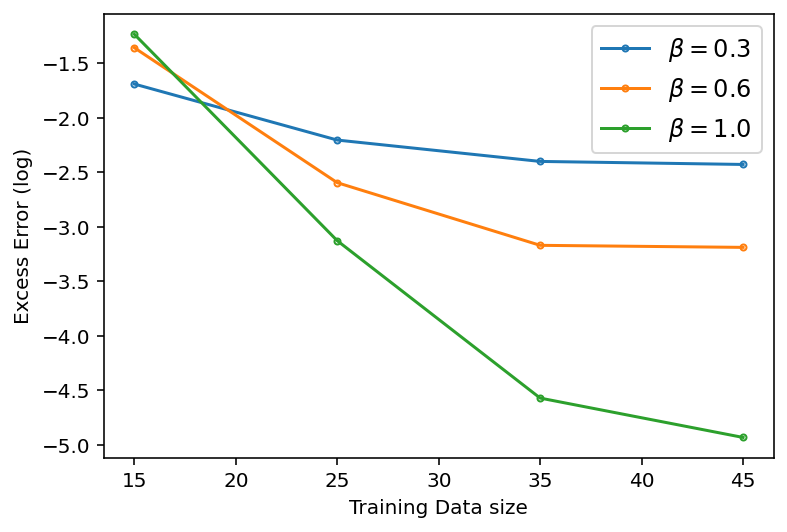}
    \vspace{-8pt}
   \caption{ekNN Classification Performance}
    \label{fig:Boltzmann_eknn}
\end{figure}

\vspace{-2pt}
\begin{figure}[t!]
    \centering
    \includegraphics[width=.86\linewidth]{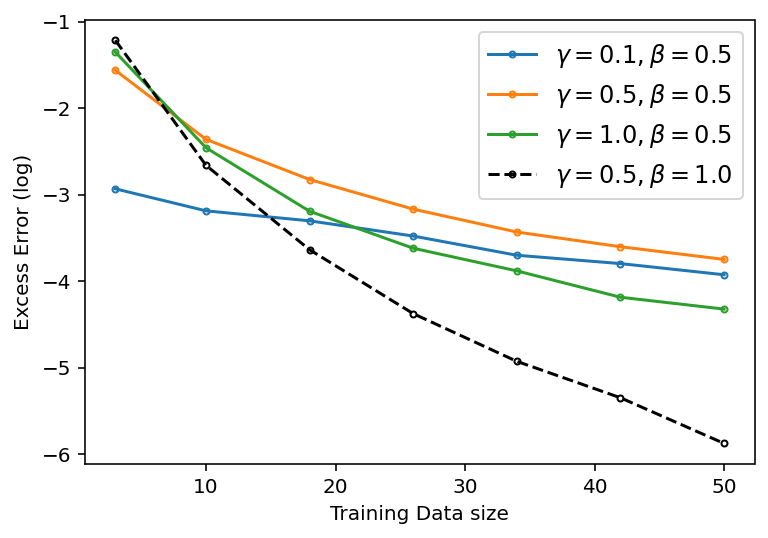}
    \vspace{-8pt}
    \caption{kNN Performance based on $\ga$'s}
    \label{fig:Boltzmann_gamma}
\end{figure}

\subsection{Classification Performance}

For empirical evaluation, 1000 
points are generated i.i.d. with $T = 0.15$ and 
$\al = 1$. We randomly shuffle them, 
use the first 25\% for testing, 
and sample from the rest for training 
(to simulate varying training data sizes). 
For ekNN, the training data sampling 
procedure follows the conditions 
in Remark \ref{rem:eknnmain}. 
All results are averaged over 
1000 shuffles.  

We evaluate excess error on the test 
set and estimate the Bayes error using 
the full data pool. For kNN, we set 
$k = C_k n^{\frac{\be + 2}{\be + C_{\up}}}$
according to Theorem \ref{thm:main_knn_rate}, 
where $C_\up = 2 + d/\up = 3$ 
and $C_k$ is chosen so that 
$k = 3$ when $n = 5$. For ekNN, we set 
$m = 5$ and $k = \tilde{C}_k \tilde{n}^{\frac{\be 
+ 2}{\be + C_{\up}}}$ according to 
Theorem \ref{thm:main_eknn_rate}, 
where $\tilde{C}_k$ is chosen in the 
same way. All $k$ values are rounded 
up to the nearest odd integers.

Figure \ref{fig:Boltzmann_knn} shows 
the excess error of kNN classifier 
with $k= 3$, $\ga = 0.5$ and various 
values of $\be$.  
The y-axis is plotted on a logarithmic 
scale, so straighter curves indicate 
rates closer to exponential. 
As $\be$ increases, the excess error 
curve becomes increasingly straight, 
which agrees with our theoretical result 
that larger $\be$ implies convergence 
closer to exponential. 

Figure \ref{fig:Boltzmann_eknn} 
shows the performance of ensemble 
kNN classifier with $k = 3$, 
$m = 5$, $\ga = 0.5$ and various 
values of $\be$. 
As $\be$ increases, the excess error  
curve becomes increasingly straight, 
suggesting convergence rates closer 
to exponential. 

Another interesting observation 
arises from comparing 
Figures \ref{fig:Boltzmann_knn} 
and \ref{fig:Boltzmann_eknn}.
The ekNN error curves appear 
straighter than those of kNN, 
suggesting a potential advantage 
of ekNN. However, neither our 
present theory nor existing ekNN 
convergence rates (Table \ref{tab:summary_rate})   
explain this gap, for they largely 
inherit the rates of kNN. Understanding 
this gap remains an interesting 
direction for future research. 

Finally, we evaluate the impact of $\ga$. 
Figure \ref{fig:Boltzmann_gamma} shows 
the kNN error based on $k = 3$ 
and $\be = 0.5$. As $\ga$ increases, 
the error decreases more rapidly, but 
the error curves do not become substantially straighter, as 
compared with the dashed curve 
obtained by increasing 
$\be$. This agrees with our theoretical 
implications: $\ga$ can accelerate 
convergence, but does not make 
its rate closer to exponential 
like $\be$.

\section{Conclusion and Discussion}
\label{sec:conclusion}

This paper introduces Boltzmann 
margin, a new margin condition 
that enables near-exponential 
convergence rates for classification. 
We demonstrate its application 
to kNN and ekNN classifiers and establish 
their first near-exponential rates, 
as well as the first strong 
consistency guarantee for ekNN. 
Extending Boltzmann margin to broader 
hypothesis classes remains an 
important direction for future research.  

Besides its theoretical contributions, 
the present study also has practical 
implications. 
For example, Lemma \ref{lem:relation_tsybakov} 
suggests that, for distributions satisfying 
a Tsybakov margin, one may expect faster 
error decay in the small-data regime 
than that implied by the margin alone. 
Furthermore, improving error rates
may not require finding a feature space 
(e.g., through data embedding) in which 
a strong Massart margin holds. Such a 
representation can be difficult to obtain 
and may involve substantial tradeoffs. Instead, 
one may identify a feature space satisfying 
Boltzmann margin while still achieving 
near-exponential error convergence. 

\section{Acknowledgment}

We thank all anonymous reviewers for 
their thorough and constructive feedback 
that improves the quality of this work.

\addtocontents{toc}{\protect\setcounter{tocdepth}{2}}

\bibliographystyle{plainnat}
\renewcommand{\bibsection}{\subsubsection*{References}}
\bibliography{reference}

\newpage

\onecolumn

\title{Near-Exponential Convergence Rates 
for kNN Classification\\ 
based on Boltzmann Margin\\ 
(Supplementary Material)}

\maketitle

\appendix

\tableofcontents

\newpage  

\section{Basic Results and Tools}

\subsection{Relations between Margins 
(Lemmas \ref{lem:relation_massart}, 
\ref{lem:relation_tsybakov}, \ref{lem:relation3})}

\textbf{Lemma \ref{lem:relation_massart}}. 
{\it A $t_*$ Massart margin implies a 
$(\al, \be, \ga)$ Boltzmann 
margin for $\al, \be, \ga > 0$ 
satisfying $\al \exp(-\frac{\ga}{t_*^\be}) = 1$. 
Conversely, for any $t_* > 0$, 
a $(1, \be, t_*^{\be})$ Boltzmann margin 
approaches a $t_*$ Massart margin 
condition as $\be \to \infty$, provided 
that $|1/2 - \eta(X)| \neq t_*$ 
almost everywhere. }

\begin{proof}
Suppose data distribution has a $t_*$ Massart margin 
and $\al, \be, \ga$ satisfies $\al e^{-\ga t_*^{-\be}} = 1$. 
For any $t \geq t_*$, 
\begin{equation}
\label{eq:prf_relation_01}
\Pr\{ |1/2 - \eta(X)|\} 
\leq t \} \leq 1 
= \al e^{-\ga t_*^{-\be}} 
\leq \al e^{-\ga t^{-\be}}. 
\end{equation}
When $t < t_*$, the above inequality 
also holds because the left side is 
zero due to Massart margin. This proves 
the first claim. 

Reversely, for any $t_* > 0$, 
suppose the distribution has a 
$(\al, \be, t_*^{\be})$ 
Boltzmann margin. Then, for any $t > 0$:  
\begin{equation}
\Pr \{ |1/2-\eta(X)| \leq t \} \leq 
\al \exp(- t_*^{\be} t^{-\be}) 
= \al \exp\left(-(t_*/t)^{\be}\right) := \Delta. 
\end{equation}
Consider three cases for $t$. 

-- If $t < t_*$, then $t_*/t > 1$ so 
that $\Delta \xrightarrow{\be \to \infty} 0$. 

-- If $t > t_*$, then $t_*/t < 1$ so 
that $\Delta \xrightarrow{\be \to \infty} \al$.

-- Case $t = t_*$ happens with probability 
zero by assumption. 

Putting all together, we have 
\begin{equation}
\lim_{\be \to \infty} 
\Pr \{ |1/2- \eta(X)| \leq t \} \leq 
\left\{ \begin{matrix}
0 & \text{if } t \leq t_*\\[.5em] 
\al & \text{if } t > t_* 
\end{matrix} \right.
\end{equation}
This proves the Boltzmann margin 
converges to a $t_*$ Massart margin. 
\end{proof}

{\bf Lemma \ref{lem:relation_tsybakov}}. 
{\it A $(\al, \be, \ga)$ Boltzmann margin 
implies a $(\frac{\al}{\ga}, \be)$ Tsybakov margin. 
Reversely, a $(\al, \be)$ Tsybakov margin 
satisfies the $(\al', \be',\ga')$ 
Boltzmann margin property (\ref{eq:boltzmann}) 
for $t \in [(\frac{\ga'}{\ln (\al'/\al)})^{1/\be'}, 
0.5]$, if $\al', \be' > 0$ and $\ga' \leq \ln (\al'/\al)$.  }

\begin{proof}
Suppose a distribution has $(\al, \be, \ga)$ 
Boltzmann margin. Then for any $t > 0$, 
\begin{equation}
\Pr\{ |1/2 - \eta(x)|\} 
\leq t \} \leq \al e^{-\ga t^{-\be}} 
\leq \al \frac{1}{\ga t^{-\be}} 
= \frac{\al}{\ga} t^{\be}, 
\end{equation}
which implies $(\al/\ga, \be)$ Tsybakov margin. 
This proves the first claim. 

Reversely, suppose a distribution has 
$(\al, \be)$ Tsybakov margin.
Let $\al', \be', \ga' > 0$ be 
any constants satisfying 
$\ga' \leq \ln (\al'/\al)$. 
Then, one can verify that 
$(\frac{\ga'}{\ln (\al'/\al)}
)^{1/\be'} \leq t \leq 0.5$ 
implies 
\begin{equation}
\al t^{\be} \leq \al \leq 
\al' \exp(-\ga' t^{-\be'}),  
\end{equation}
where the first inequality follows 
$t^\be \leq 1$ and the second 
follows the lower bound of $t$. 
Hence the $(\al', \be',\ga')$ 
Boltzmann margin. Note that $\be$ 
is not included in the condition 
on $t$ because we set $t \leq 0.5$, 
in which case $t^{\be} \leq 1$ 
for any $\be > 0$. 
\end{proof}

{\bf Lemma \ref{lem:relation3}}. 
{\it A $(\al, \be, \ga)$ Boltzmann margin implies 
a  $(\al', \be', \ga')$ Boltzmann margin whenever 
(i) $\al' \geq \al$, (ii) $\be' \geq \be$ or (iii) $\ga' \leq \ga$, with the remaining parameters fixed.  }

\begin{proof}
This property follows the definition. 
Recall $(\al, \be,\ga)$ Boltzmann margin means 
$\Pr\{ |1/2 - \eta(x)|\} 
\leq t \} \leq \al e^{-\ga t^{-\be}}$. 
Apparently, the right side increases 
if $\al$ or $\be$ increases, or $\ga$ decreases. 
\end{proof}

\subsection{Relation between 
Excess Error and Error Probability}

In this section, we present some 
useful relations between excess error 
and error probability, which will be 
applied later. 

\begin{lemma}
\label{lem:exer2geer}
Let $\hat{y}$ be any classifier 
trained on sample $S$. For any fixed 
$S$, there is 
\begin{equation}
er(\hat{y}) - er(\hat{y}_*) \leq 
{\Pr}_{X} \{ \hat{y}(X) \neq \hat{y}_*(X) \}.
\end{equation}
Further, over the random choice 
of $S$, there is 
\begin{equation}
\E_S [er(\hat{y})] 
- er(\hat{y}_*) \leq {\Pr}_{X, S} 
\{ \hat{y}(X) \neq \hat{y}_*(X) \}; 
\end{equation}
and for any $t > 0$, 
\begin{equation}
{\Pr}_S \{ er(\hat{y})  
- er(\hat{y}_*) > t\} 
\leq \frac{1}{t} {\Pr}_{X, S}\{ 
\hat{y}(X) \neq \hat{y}_*(X) \}.   
\end{equation}
\end{lemma}
\begin{proof}
Let $Y$ be the label of $X$. 
Since $\hat{y}(X) \neq Y$ 
implies $\hat{y}(X) \neq \hat{y}_*(X)$ 
or $\hat{y}_*(X) \neq Y$, there is
\begin{equation}
{\Pr}_{X,Y}\{ \hat{y}(X) \neq Y \} \leq 
{\Pr}_{X} \{ \hat{y}(X) \neq \hat{y}_*(X) \}
+ {\Pr}_{X,Y}\{ \hat{y}_*(X) \neq Y\},   
\end{equation}
which implies 
\begin{equation}
\label{eq:lem_basic_prf_01}
{\Pr}_{X}\{ \hat{y}(X) \neq \hat{y}_*(X) \} 
\geq {\Pr}_{X,Y}\{ \hat{y}(X) \neq Y \} 
- {\Pr}_{X,Y} \{ \hat{y}_*(X) \neq Y \} 
= er(\hat{y}) - er(\hat{y}_*). 
\end{equation}
This proves the first claim. The second 
claim follows by taking expectation w.r.t. 
$S$ on both sides of (\ref{eq:lem_basic_prf_01}), 
and the third claim follows by further 
employing the Markov inequality (based on 
the fact that $er(\hat{y}) - 
er(\hat{y}_*) \geq 0$). 
\end{proof}

The above lemma shows 
${\Pr}_{X}\{ \hat{y}(X) \neq \hat{y}_*(X)\}$ 
is key for bounding excess error. 
The following lemma provides an 
observation that is useful for bounding 
${\Pr}_{X}\{ \hat{y}(X) \neq \hat{y}_*(X)\}$. 

\begin{lemma}
\label{lem:cond4geer}    
Let $\hat{y}$ be a classifier 
defined as $\hat{y}(x) = \I\{\hat{y}(x; S) 
\geq 1/2\}, \forall x \in \X$. 
Then, almost surely, any $x \in X$ 
satisfying $|\hat{\eta}(x; S) - \eta(x)| 
\leq |1/2 - \eta(x)|$ also satisfies 
$\hat{y}(x) = \hat{y}_*(x)$. 
\end{lemma}
\begin{proof}
It is clear that $\hat{y}(x) = \hat{y}_*(x)$ 
if $\hat{\eta}(x; S)$ and $\eta(x)$ fall
on the same side of 1/2, which is guaranteed 
if $|\hat{\eta}(x; S) - \eta(x)| 
\leq |1/2 - \eta(x)|$. This proves the lemma. 
\end{proof}

\subsection{Other Tools}

We borrow the following kNN regression bound 
from prior work. The original 
bound is uniform and very strong, 
but we only need its implied pointwise 
bound. The following is a paraphrase of 
\cite[Proposition 1]{xue2018achieving}. 

\begin{proposition} 
\label{lem:xue2}
Let $\V$ be the VC dimension of the class 
of balls on $\X$. 
Let $\delta < 1$ and 
$k = O((\ln \frac{2|Z|}{\delta})^{
\frac{d}{2 \up + d}} (C_d |Z|)^{\frac{2\up}{2 
\up + d}})$. With probability at least $1 - 
2 \delta$ over the random sampling of $Z$, 
we have simultaneously for all $x \in \X$: 
\begin{equation}
\label{eq:lem:xue2_01}
| \hat{\eta}(x; Z) - \eta(x) | 
\leq C \left( \frac{\V \ln(2|Z|/\delta)}{
C_d |Z|}\right)^{\frac{\up}{2\up+d}},   
\end{equation}
where $C$ is a function of $\up, \la$, 
and $C_d$ is a constant depending on intrinsic 
data dimension $d$. 
\end{proposition}
\begin{corollary} 
\label{cor:lem_xue2}
Let $C_\up = 2 + \frac{d}{\up}$ 
and $C' = \frac{C_d}{\V C^{C_\up}}$. 
Proposition \ref{lem:xue2} implies 
for any $t > 0$, if $k = O(t^{\frac{d}{\up}} |Z| 
C_d^{\frac{2 \up + 2d}{2\up + d}})$ then  
\begin{equation}
\label{eq:cor:xue2_01}
{\Pr}_{X, Z} \{| \hat{\eta}(X; Z) - 
\eta(X) | > t\} \leq 2 |Z| 
\exp\left(- C' |Z| t^{C_\up}\right).  
\end{equation}
\end{corollary} 
\begin{proof}
Proposition \ref{lem:xue2} implies 
for any $x \in X$, (\ref{eq:lem:xue2_01}) 
holds with high probability. 
Setting the right side of (\ref{eq:lem:xue2_01})
to $t$ implies ${\Pr}_{Z} \{| \hat{\eta}(x; Z) - 
\eta(x) | > t\} \leq 2 |Z| 
\exp\left(- \frac{C_d}{\V} |Z| 
\left(\frac{t}{C}\right)^{C_\up}\right)$. 
Taking expectation w.r.t. $x$ on both sides 
gives the corollary. 
\end{proof}

We will also apply the classic Borel-Cantelli 
Lemma to prove storng consistency. It is quoted 
as follows. 
\begin{lemma}
\label{lem:borelcantelli}
Let $E_1, E_2, \ldots$ be a sequence of 
events. If $\sum_{i=1}^\infty \Pr\{E_i\} < \infty$, 
then $\Pr\{{\lim\sup}_{i \to \infty} E_i\} = 0$. 
\end{lemma}

\section{Results for kNN Classification}

\addcontentsline{toc}{subsection}{B.1\quad  
   Theorem \ref{thm:main_knn_rate}}
{\bf Theorem \ref{thm:main_knn_rate}}. 
{\it Suppose data distribution has a 
$(\al, \be, \ga)$ Boltzmann margin. 
Let $C_\up = 2 + d/\up$. 
If $k = O(n^{\frac{\be+2}{\be+C_\up}})$, 
then for $n \geq (C' \ln n)^{1+C_\up/\be}$: 
\begin{equation*}
||\hat{y} - \hat{y}_*||_e
\leq C_\al \exp\left(- C_\ga
 n^{\frac{\be}{\be + C_\up}} \right),     
\end{equation*}    
where $C_\al = 2 + \al$, $C_\ga = \min(\ga, 1/C')$ 
and $C'>0$ is a constant depending on 
$\la, \up$ and the structure of $\X$. 
}

\begin{proof}
By Lemma \ref{lem:exer2geer} and Lemma 
\ref{lem:cond4geer}, for any $t > 0$, 
\begin{equation}
\label{eq:prf_lem_exer2geer}
\begin{split}
||\hat{y} - \hat{y}_*||_e & \leq 
{\Pr}_{X, S}\{ \hat{y}(X) \neq \hat{y}_*(X) \}\\[.5em]  
& \leq {\Pr}_{X, S}\{|\hat{\eta}(X; S) - \eta(X)| 
> |1/2 - \eta(X)|\}\\[.5em] 
& \leq {\Pr}_{X, S}\{|\hat{\eta}(X; S) - \eta(X)| > t\}
+ {\Pr}_{X}\{|1/2 - \eta(X)| \leq t\}, 
\end{split}
\end{equation}
where the last inequality is by 
case-studying $|1/2 - \eta(X)|$ 
based on $t$.     

On the right side of (\ref{eq:prf_lem_exer2geer}), 
the second term can be bounded based 
on Boltzmann margin i.e, 
\begin{equation}
\label{eq:prf_01_blz}
{\Pr}_{X}\{|1/2 - \eta(X)| \leq t\} 
\leq \al \exp(-\ga t^{-\be}).
\end{equation}

The first term can be bounded using Corollary \ref{cor:lem_xue2}. This corollary sets $C_{\up} 
= 2 +d/\up$  and $C' = \frac{C_d}{\V C^{C_\up}}$, where 
$C_d>0$ is a constant depending on $d$ 
and $C$ is the constant in Lemma \ref{lem:xue2} 
that depends on $\la$ and $\up$. (This means 
$C'$ depends on $\la, \up, d, \V$. )  
The corollary says if $k = O(t^{\frac{d}{\up}} 
n C_{d}^{\frac{2 \up + 2d}{2\up + d}})$ 
 then 
\begin{equation}
\label{eq:prf_01}
{\Pr}_{X, S} \{ |\hat{\eta}(X; S) - \eta(X)| > t \}    
\leq 2 n \exp(- C' n t^{C_{\up}}).  
\end{equation}
Plugging this and (\ref{eq:prf_01_blz}) back to 
(\ref{eq:prf_lem_exer2geer}), we have 
\begin{equation}
\label{eq:thm_prf_03}
{\Pr}_{X, S} \{ \hat{y}(X) \neq \hat{y}_*(X) \} 
\leq 2n \exp(- C' n t^{C_\up}) 
+ \al \exp(-\ga t^{-\be}) := \Delta.        
\end{equation}
To merge the two terms in $\Delta$, 
we first assume $n \geq \frac{2}{C'} t^{-C_\up} \ln n$, 
which is realizable for large enough $n$ 
and a proper choice of $t$. Then, we have 
$n \leq \exp(\frac{1}{2} C' n t^{C_\up})$ and  
thus 
\begin{equation}
\Delta \leq 2 \exp(- \frac{1}{2} C' 
n t^{C_\up}) + \al \exp(-\ga t^{-\be}). 
\end{equation}
Set $n t^{C_\up} = t^{-\be}$. 
Then $t = n^{-\frac{1}{C_\up + \be}}$ and 
the above implies 
\begin{equation}
\label{eq:prf_0012}
\Delta \leq (2 + \al) 
\exp\left(- \min(\frac{1}{2}C', \ga) 
\cdot n^{\frac{\be}{\be + C_\up}} \right).      
\end{equation}
Also, based on the above choice of $t$, the 
conditions on $n$ and $k$ 
become $n \geq (\frac{2}{C'} 
\ln n)^{\frac{C_\up + \be}{\be}}$ and 
$k = O(n^{\frac{2 +\be}{C_\up + \be}})$. 
The theorem is proved by 
substituting $\frac{2}{C'}$ by $C'$, 
and noting that $C'$ depends on $\la, \up$ 
as well as $d,\V$ (structure of $\X$). 
\end{proof}

\addcontentsline{toc}{subsection}{B.2\quad  
   Remark \ref{rem:thm_knn_rate}}
{\bf Remark \ref{rem:thm_knn_rate}}. 
{\it Theorem \ref{thm:main_knn_rate} implies 
that, for $n \geq 1$: 

(i) If $k = O(n^{\frac{\be+2}{\be+C_\up}})$, 
then there exists a constant $\tilde{C}_\al > 0$ such 
that $||\hat{y} - \hat{y}_*||_e \leq \tilde{C}_{\al} 
\exp\left(-  C_\ga n^{\frac{\be}{\be + C_\up}} \right)$. 

(ii) Take $\be \to \infty$. Then 
$||\hat{y} - \hat{y}_*||_e 
\leq C_\al' \exp\left(- C_\ga n \right)$ 
for $k = O(n)$,  
where $C_\al' = \lim_{\be \to \infty} \tilde{C}_\al 
< \infty$.  
}

\begin{proof}
To prove (i), we start with arguments 
(\ref{eq:thm_prf_03}) and (\ref{eq:prf_0012}) 
in the proof of Theorem \ref{thm:main_knn_rate}, 
which imply  
\begin{equation}
\label{eq:rem_thm1_prof_o1}
{\Pr}_{X, S} \{ \hat{y}(X) \neq \hat{y}_*(X) \} 
\leq (2 + \al) \exp\left(- \min(\frac{1}{2}C', \ga) 
\cdot n^{\frac{\be}{\be + C_\up}} \right)
\end{equation}
when $n \geq (\frac{2}{C'} \ln n)^{
\frac{C_\up + \be}{\be}}$. This constraint can 
be satisfied for all $n \geq N$, where $N > 0$ is 
a constant depending on $\la, \up, d, \V, \be$. 
In other words, when $n \geq N$, we have 
(\ref{eq:rem_thm1_prof_o1}). 

Set 
\begin{equation}
\label{eq:prf_CN}
C_N := (2 + \al) \exp\left(- \min(\frac{1}{2}C', \ga) 
\cdot N^{\frac{\be}{\be + C_\up}} \right).     
\end{equation}
Then, when $n \leq N$, there is 
\begin{equation}
\label{eq:prf_0013}
\begin{split}
{\Pr}_{X, S} \{ \hat{y}(X) \neq \hat{y}_*(X) \}  
\leq 1  & = \frac{1}{C_N} (2 + \al) 
\exp\left(- \min(\frac{1}{2}C', \ga) 
\cdot N^{\frac{\be}{\be + C_\up}} \right) \\
& \leq \frac{1}{C_N} (2 + \al) 
\exp\left(- \min(\frac{1}{2}C', \ga) 
\cdot n^{\frac{\be}{\be + C_\up}} \right).       
\end{split}
\end{equation}
Combining (\ref{eq:rem_thm1_prof_o1}) 
and (\ref{eq:prf_0013}) with 
$\tilde{C}_\al := \max(1, C_N^{-1}) \cdot 
(2 + \al)$, we have 
\begin{equation}
\label{eq:prf_0014}
{\Pr}_{X, S} \{ \hat{y}(X) \neq \hat{y}_*(X) \}  
\leq \tilde{C}_{\al} 
\exp\left(- \min(\frac{1}{2}C', \ga) 
\cdot n^{\frac{\be}{\be + C_\up}} \right).       
\end{equation}
This proves claim (i) in the remark. 

To prove claim (ii), it is tempting 
to simply take the limit of $\be$ on both sides 
of (\ref{eq:prf_0014}). However, we need to 
first make sure $\lim_{\be \to \infty} \tilde{C}_\al 
< \infty$ which is not obvious since $ \tilde{C}_\al$ 
also depends on $\be$. 

Recall $\tilde{C}_\al = \max(1, C_N^{-1}) \cdot 
(2 + \al)$ where $C_N$ is set in (\ref{eq:prf_CN}), 
it is clear a sufficient condition for $\lim_{\be \to \infty} \tilde{C}_\al < \infty$ is $\lim_{\be \to \infty} 
N^{\frac{\be}{\be + C_\up}} < \infty$. 
This is true for the following reason. 

First, as $\be \to \infty$, $N$ remains bounded 
because now we need every $n \geq N$ to satisfy 
\begin{equation}
n \geq \lim_{\be \to \infty} (\frac{2}{C'} 
\ln n)^{\frac{C_\up + \be}{\be}} 
= \frac{2}{C'} \ln n,     
\end{equation}
which is easy to fulfill by a finite $N$. 
Let $N_*$ be such an integer. Then, 
$\lim_{\be \to \infty} N^{\frac{\be}{\be + C_\up}} 
= N_*$. 

Now we can safely take the limit of $\be$ on both sides 
of (\ref{eq:prf_0014}), which proves claim (ii). 
\end{proof}

\addcontentsline{toc}{subsection}{B.3\quad  
   Corollary \ref{cor:main_knn_rate_v2}}
{\bf Corollary \ref{cor:main_knn_rate_v2}}. 
{\it Suppose data distribution has 
$(\al, \be, \ga)$ Boltzmann margin 
and $k = O(n^{\frac{\be+2}{\be+C_\up}})$. 
Then for any $\delta > 0$, the following 
holds for $n \geq 1$: 
\begin{equation*}
er(\hat{y}) - er(\hat{y}_*) \leq 
\tilde{C}_\al \delta^{-1} \exp\left(- 
C_\ga n^{\frac{\be}{\be + C_\up}} \right),  
\end{equation*}
with probability at least $1 - \delta$ 
over the random choice of $S$.
}

\begin{proof}
Combining Lemma \ref{lem:exer2geer} 
and argument (\ref{eq:prf_0014}) 
for Remark \ref{rem:thm_knn_rate}, 
for any $t > 0$ and $n \geq 1$, 
\begin{equation}
{\Pr}_S \{ er(\hat{y}) - er(\hat{y}_*) > t\} \leq 
\frac{1}{t} \tilde{C}_\al \exp(- C_\ga 
n^{\frac{\be}{\be + C_\up}}),  
\end{equation}
Setting the right side to $\delta$ and 
solving for $t$ proves the corollary. 
\end{proof}

\addcontentsline{toc}{subsection}{B.4\quad  
   Corollary \ref{cor:main_knn_rate_v3}}
{\bf Corollary \ref{cor:main_knn_rate_v3}}. 
{\it Suppose data distribution has 
a $(\al', \be')$ Tsybakov margin  
and let $\al, \be, \ga>0$ be 
constants satisfying 
$\ga \leq \ln(\al/\al')$. 
If $k = O(n^{\frac{2 +\be}{C_\up + \be}})$, 
then 
\begin{equation*}
||\hat{y} - \hat{y}_*||_e
\leq \tilde{C}_{\al} \exp(- C_{\ga} 
n^{\frac{\be}{\be + C_\up}}),   
\end{equation*}
for $2 \leq n \leq \left( \ga^{-1} 
\ln(\al/\al') \right)^{1 + C_\up/\be}$.   }

\begin{proof}
By Lemma \ref{lem:relation_tsybakov}, the 
distribution satisfies the $(\al, \be,\ga)$ 
Boltzmann margin property when 
$t \in [(\ga/\ln (\al/\al'))^{1/\be}, 0.5]$. 

Revisit the proof of Theorem \ref{thm:main_knn_rate} 
based on this implied Boltzmann margin. 
It sets $t = n^{\frac{-1}{C_\up + \be}}$, 
which implies 
\begin{equation}
2 \leq n \leq 
\left( \frac{1}{\ga} \ln(\al/\al')
\right)^{1 + C_\up/\be}. 
\end{equation}
This proves the corollary. 
\end{proof}

\addcontentsline{toc}{subsection}{B.5\quad  
   Corollary \ref{cor:main_knn_rate_v4}}
{\bf Corollary \ref{cor:main_knn_rate_v4}}. 
{\it Let $k = O(n)$. If data distribution 
has a $t_*$ Massart margin, then 
\begin{equation}
\label{eq:massart_exprate}
||\hat{y} - \hat{y}_*||_e \leq 
C \exp(- C_* n), 
\end{equation}    
for all $n \geq C_*^{-1} \ln n$,    
where $C$ and $C_* \propto t_*^{C_\up}$ 
are positive constants. 
Moreover, for every $t_* > 0$, 
there exists a data distribution 
satisfying the $t_*$ Massart margin 
for which (\ref{eq:massart_exprate}) 
holds for all $n \geq 1$ (possibly 
with different constants).   
}

\begin{proof}

By Lemma \ref{lem:relation_massart}, the 
distribution also has $(\al, \be, \ga)$ 
Boltzmann margin as long as $\al \exp(-\ga t_*^{-\be}) = 1$. 
Revisit the proof of Theorem \ref{thm:main_knn_rate} 
based on this implied Boltzmann margin. 

Arguments (\ref{eq:prf_lem_exer2geer}) 
(\ref{eq:prf_01}) imply that, if
$k = O(t^{\frac{d}{\up}} n 
C_{d}^{\frac{2 \up + 2d}{2\up + d}})$,  
then for any $t > 0$: 
\begin{equation}
||\hat{y} - \hat{y}_*||_e \leq 
{\Pr}_{X, S}\{ \hat{y}(X) \neq \hat{y}_*(X) \} 
\leq 2n \exp(- C' n t^{C_\up}) 
+ {\Pr}_{X}\{|1/2 - \eta(X)| \leq t\}.  
\end{equation}
Set $t = t_*$. Then, the right probability 
vanishes due to Massart margin and the above 
becomes: if $k = O(t_*^{\frac{d}{\up}} 
n C_{d}^{\frac{2 \up + 2d}{2\up + d}})$, then  
\begin{equation}
||\hat{y} - \hat{y}_*||_e
\leq 2n \exp(- C' t_*^{C_\up} n) 
\leq 2 \exp(- \frac{1}{2}C' t_*^{C_\up} n),  
\end{equation}
where the second inequality holds 
when $n \leq \exp(\frac{1}{2} C' n t_*^{C_\up})$. 
This proves the first claim of the corollary. 

To prove the second claim, 
start with a data distribution that 
has $(1, \be, t_*^{\be})$ Boltzmann margin 
for any given $t_* > 0$. 

Set $\be \to \infty$. Then two things happen: 
Lemma \ref{lem:relation_massart} suggests this data distribution gains a $t_*$ Massart margin, and 
Remark \ref{rem:thm_knn_rate} says 
$||\hat{y} - \hat{y}_*||_e 
\leq C_\al' \exp\left(- C_\ga n \right)$ 
for $k = O(n)$. For the second event, 
we should remark that $\lim_{\be \to \infty} C_\ga < \infty$ because $C_\ga = \min(\frac{1}{2}C', \ga)$ 
and $C'$ does not depend on margin parameters. 
Combining them proves the second claim. 
\end{proof}

\section{Results for Ensemble kNN Classification}

\addcontentsline{toc}{subsection}{C.1\quad  
   Theorem \ref{thm:main_eknn_rate}}
{\bf Theorem \ref{thm:main_eknn_rate}}. 
{\it Suppose data distribution has a 
$(\al, \be, \ga)$ Boltzmann margin. 
If $k = O(\tilde{n}^{\frac{\be + 2}{\be 
+ C_\up}})$, then 
\begin{equation}
\label{eq:thm_eknn}
||\hat{y}_e - \hat{y}_*||_e
\leq C_{\al}
\exp(- C_\ga \tilde{n}^{\frac{\be}{\be + C_\up}}),     
\end{equation}  
for $\tilde{n} \geq 
\left( 2 \max \{ \ln \tilde{n} 
, 2 \ln m \}/C' \right)^{1+C_\up/\be}$. 

Moreover, if $m = O(\sqrt{\tilde{n}})$, 
then (\ref{eq:thm_eknn}) holds for 
all $\tilde{n} \geq 1$, with $C_\al$ 
replaced by a possibly different 
constant $\tilde{C}_{\al} > 0$.
}

\begin{proof}
Recall $\hat{y}_e$ is trained on $\tilde{S}$, 
which is a collection of $m$ samples $S_1, \ldots, S_m$ 
each containing $\tilde{n}$ i.i.d. sampled points. 

By Lemma \ref{lem:exer2geer} and Lemma 
\ref{lem:cond4geer}, for any $t > 0$, 
\begin{equation}
\label{eq:prf_lemeknn_exer2geer}
\begin{split}
||\hat{y}_e - \hat{y}_*||_e & \leq 
{\Pr}_{X, \tilde{S}} \{ \hat{y}_e(X) \neq \hat{y}_*(X) \}\\[.2em]  
& \leq {\Pr}_{X, \tilde{S}}\{|\hat{\eta}_e(X; 
\tilde{S}) - \eta(X)| > t\}
+ {\Pr}_{X}\{|1/2 - \eta(X)| \leq t\}. 
\end{split}
\end{equation}
The second term can be bounded using Boltzmann 
margin. For the first term, 
\begin{equation}
\label{eq:prf_lemeknn_exer2geer_02}
\begin{split}
{\Pr}_{X, \tilde{S}}\{|\hat{\eta}_e(X; 
\tilde{S}) - \eta(X)| > t\} 
& = {\Pr}_{X, \tilde{S}} \left\{\frac{1}{m} 
\left|\sum_{j=1}^m \hat{\eta}(X; S_j) 
- \eta(X)\right| > t\right\} \\ 
& \leq \sum_{j=1}^m {\Pr}_{X, S_j} 
\left\{ \left|\hat{\eta}(X; S_j) 
- \eta(X) \right| > t\right\}\\
& = m \cdot {\Pr}_{X, S_1} \left\{ 
\left|\hat{\eta}(X; S_1) - 
\eta(X) \right| > t\right\},  
\end{split}
\end{equation}
where the inequality follows a union bound 
and the last equality is based on the 
i.i.d. sampling of $S_1, \ldots, S_m$. 

The right side of (\ref{eq:prf_lemeknn_exer2geer_02})  
can be bounded using Corollary \ref{cor:lem_xue2}. 
Recall $C_{\up} = 2 +d/\up$  
and $C' = \frac{C_d}{\V C^{C_\up}}$. 
The corollary says 
if $k = O(t^{\frac{d }{\up}} \tilde{n} 
C_d^{\frac{2 \up + 2d}{2\up + d}})$ then
${\Pr}_{X, S_1} \left\{ 
\left|\hat{\eta}(X; S_1) - 
\eta(X) \right| > t\right\} \leq 
2 \tilde{n} \exp\left(- C' 
\tilde{n} t^{C_\up}\right)$. 
Plugging this back to (\ref{eq:prf_lemeknn_exer2geer_02}), 
we have 
\begin{equation}
\begin{split}
{\Pr}_{X, \tilde{S}}\{|\hat{\eta}_e(X; 
\tilde{S}) - \eta(X)| > t\} \leq 
2 m \tilde{n} \exp\left(- C' 
\tilde{n} t^{C_\up}\right)
\leq 2 m \exp\left(- \frac{1}{2} 
C' \tilde{n} t^{C_\up}\right)
\leq 2 \exp\left(- \frac{1}{4} 
C' \tilde{n} t^{C_\up}\right), 
\end{split}
\end{equation}
where the second inequality holds when 
$\tilde{n} \leq \exp(\frac{1}{2} C' 
\tilde{n} t^{C_\up})$ and the third 
holds when $m \leq \exp\left( 
\frac{1}{4} C' \tilde{n} t^{C_\up}\right)$. 

Plugging the above back to 
(\ref{eq:prf_lemeknn_exer2geer}), we have 
\begin{equation}
{\Pr}_{X, \tilde{S}}\{ \hat{y}_e(X) \neq \hat{y}_*(X)\} 
\leq 2 \exp\left(- \frac{1}{4} 
C' \tilde{n} t^{C_\up}\right)
+ \al \exp(-\ga t^{-\be}).  
\end{equation}
To merge terms, set $\tilde{n} 
t^{C_\up} = t^{-\be}$. Then 
$t = \tilde{n}^{-\frac{1}{\be + C_\up}}$ 
and the above implies 
\begin{equation}
\label{eq:thm:proof_boundfinal}
{\Pr}_{X, \tilde{S}}\{ \hat{y}_e(X) \neq \hat{y}_*(X)\} 
\leq (2+\al) \exp\left(- \min\left(\frac{1}{4}C', \ga 
\right) \cdot \tilde{n}^{\frac{\be}{\be + C_\up}} \right).  
\end{equation}
Based on this choice of $t$, 
the above constraints on $\tilde{n}$ and $n$ 
can be merged into 
\begin{equation}
\label{eq:prf_eknn_tnbound}
\tilde{n} \geq 
\left(\frac{2}{C'} \max \{ \ln \tilde{n} 
, 2 \ln m \} \right)^{1+C_\up/\be}.    
\end{equation}
This proves the first claim of the theorem. 

The second claim follows from the fact that, 
if $m = O(\sqrt{\tilde{n}})$, 
then a sufficient condition for (\ref{eq:prf_eknn_tnbound}) to 
hold is that $\tilde{n} \geq (C'' \ln \tilde{n})^{1 + C_\up \be}$ 
for some constant $C'' > 0$. It is clear this 
new constraint can be satisfied for large 
enough $\tilde{n}$. This means there exists an 
integer $\tilde{N} > 0$ such that $n \geq \tilde{N}$ implies 
$\tilde{n} \geq (C'' \ln \tilde{n})^{1 + C_\up \be}$ 
and thus (\ref{eq:thm:proof_boundfinal}). Set 
\begin{equation}
C_{\tilde{N}} = (2+\al) 
\exp\left(- \min\left(\frac{1}{4}C', \ga \right) 
\cdot \tilde{N}^{\frac{\be}{\be + C_\up}} \right).   
\end{equation}
Then, when $n \leq \tilde{N}$, we have 
\begin{equation}
\label{eq:thm:proof_boundfinal2}
\begin{split}
{\Pr}_{X, \tilde{S}}\{ \hat{y}_e(X) \neq \hat{y}_*(X)\} 
\leq 1 & = \frac{1}{C_{\tilde{N}}} 
(2+\al) 
\exp\left(- \min\left(\frac{1}{4}C', \ga \right) 
\cdot \tilde{N}^{\frac{\be}{\be + C_\up}} \right)\\
& \leq \frac{1}{C_{\tilde{N}}} (2+\al) 
\exp\left(- \min\left(\frac{1}{4}C', \ga \right) 
\cdot \tilde{n}^{\frac{\be}{\be + C_\up}} \right).
\end{split}
\end{equation}
Finally, combining (\ref{eq:thm:proof_boundfinal}) 
and (\ref{eq:thm:proof_boundfinal2}) with 
$\tilde{C}_\al := \max\{2+\al, \frac{1}{C_{\tilde{N}}} 
(2+\al) \}$ proves the second claim. 
\end{proof}

\begin{remark}
\label{rem:thm_eknn_tool}
Theorem \ref{thm:main_eknn_rate} implies that, 
if data distribution has $(\al, \be, \ga)$ 
Boltzmann margin and $k = O(\tilde{n}^{\frac{\be + 2}{\be 
+ C_\up}})$, then 
\begin{equation}
{\Pr}_{X, \tilde{S}}\{ \hat{y}_e(X) 
\neq \hat{y}_*(X)\} \leq \tilde{C}_\al \exp(- 
C_\ga \tilde{n}^{\frac{\be}{\be +C_\up}})    
\end{equation}
for either (i) 
$m = O(\sqrt{\tilde{n}})$ and $\tilde{n} \geq 1$, 
or (ii) $m = \tilde{O}(e^{\tilde{n}})$ 
and $\tilde{n} \geq 
\left( 2 \max \{ \ln \tilde{n} 
, 2 \ln m \}/C' \right)^{1+C_\up/\be}$. 
\end{remark}

\addcontentsline{toc}{subsection}{C.2\quad  
   Remark \ref{rem:eknnmain}}
{\bf Remark \ref{rem:eknnmain}}. 
{\it
Theorem \ref{thm:main_eknn_rate} 
also applies when 
(i) a training set $S$ is sampled 
i.i.d. from the population, (ii) 
subsets $S_1, \ldots, S_m$ 
are drawn i.i.d. from $S$, 
and (iii) in each subset, points 
are sampled without replacement. }

\begin{proof}
Consider the alternative setting. 
Let $S$ be a training set, whose elements are sampled i.i.d. from the population. Let $S_1, \ldots, S_m$ subsets of $S$ that are independently sampled, and points in each 
subset are sampled without replacement. 
Let $\tilde{S}$ be the collection of 
these subsets. 
Let $\hat{y}_b$ be the ensemble kNN classifier built on above $\tilde{S}$. Its expected excess error is 
\begin{equation}
|| \hat{y}_{b} - \hat{y}_* ||_{e} 
= \mathbb{E}_{\tilde{S}, S} [er( 
\hat{y}_{b})] - er( \hat{y}_*).    
\end{equation}
We will show this error has the same bound 
presented in Theorem \ref{thm:main_eknn_rate}. 

First, by the arguments in Lemma \ref{lem:exer2geer} proof, we have 
\begin{equation}
er(\hat{y}) - er(\hat{y}_*) 
\leq {\Pr}_{X} \{ \hat{y}(X) \neq 
\hat{y}_*(X) \}.
\end{equation}
Taking expectation on both sides of 
the above w.r.t. $S$ and $\tilde{S}$, 
we have 
\begin{equation}
\label{eq:rem_cmp_01}
|| \hat{y}_{b} - \hat{y}_* ||_{e} 
\leq \Pr_{X, \tilde{S}, S} \{ \hat{y}_b 
(X) \neq \hat{y}_* (X) \} = \mathbb{E}_S \Pr_{X,\tilde{S} \mid S} \{ \hat{y}_b(X) 
\neq {\hat{y}}_* (X) \}.
\end{equation}
Second, by the arguments in Lemma 
\ref{lem:cond4geer} proof, we have
\begin{equation}
\label{eq:rem_cmp_02}
\mathbb{E}_S \Pr_{X,\tilde{S} \mid S} 
\{ \hat{y}_b(X) \neq \hat{y}_*(X) \} \leq \mathbb{E}_S \Pr_{X,\tilde{S} \mid S} \{ |\hat{\eta}_b(X; \tilde{S}) - \eta(X)|>t 
\} + \Pr_X \{ |1/2-\eta(X)| \leq t \}.
\end{equation}

Next, by following argument (\ref{eq:prf_lemeknn_exer2geer_02}) 
in Theorem 5.1 proof, we have 
\begin{equation}
\begin{split} 
\E_S \Pr_{X,\tilde{S} \mid S} 
\{ |\hat{\eta}_b(X; \tilde{S}) - \eta(X)|>t\}
& = \E_S \Pr_{X,\tilde{S} \mid S} 
\left\{\frac{1}{m} 
\left|\sum_{j=1}^m \hat{\eta}(X; S_j) 
- \eta(X)\right| > t\right\} \\
&\leq \mathbb{E}_S \sum_{j=1}^m \Pr_{X,S_j 
\mid S} \{|\hat{\eta}(X; S_j) - \eta(X)| 
> t \} \\ 
& = \sum_{j=1}^m \mathbb{E}_S \Pr_{X,S_j 
\mid S} \{ |\hat{\eta}(X; S_j) - 
\eta(X)| > t \} \\
&= \sum_{j=1}^m \Pr_{X,S_j} \{ |\hat{\eta}(X; 
S_j) - \eta(X)| > t \} \\[.1em]
& = m \Pr_{X,S_1} \{ |\hat{\eta}(X; S_1) 
- \eta(X)| > t \}. 
\end{split}    
\end{equation} 
In above, $\Pr_{X, S_j \mid S}$ only 
treats $S_j$ as a random subset given $S$, 
while $\Pr_{X, S_j}$ treats $S_j$ as an 
i.i.d. sample of the population (because 
its elements are sampled from $S$ without replacement). 
The last equality holds because $S_j$'s 
are sampled i.i.d. from $S$ and thus now 
treated as i.i.d. samples of the population. 
Plugging the above back to (\ref{eq:rem_cmp_02}) 
then (\ref{eq:rem_cmp_01}), we have 
\begin{equation}
\label{eq:rem_cmp_01}
|| \hat{y}_{b} - \hat{y}_* ||_{e} 
\leq m \Pr_{X,S_1} \{ |\hat{\eta}(X; S_1) 
- \eta(X)| > t \} 
+ \Pr_X \{ |1/2-\eta(X)| \leq t \}.
\end{equation}
Comparing this bound with the 
bound (\ref{eq:prf_lemeknn_exer2geer} 
+ \ref{eq:prf_lemeknn_exer2geer_02}) 
presented in Theorem 
\ref{thm:main_eknn_rate} proof i.e., 
\begin{equation}
||\hat{y}_e - \hat{y}_*||_e 
\leq m \cdot \Pr_{X, S_1} \left\{ 
\left|\hat{\eta}(X; S_1) - 
\eta(X) \right| > t\right\} 
+ \Pr_{X}\{|1/2 - \eta(X)| \leq t\},  
\end{equation}
we see $\hat{y}_{b}$ has the 
same excess error bound as 
$\hat{y}_e$. It is also noted 
the rest analysis on $\Pr_{X, S_1} \left\{ 
\left|\hat{\eta}(X; S_1) - 
\eta(X) \right| > t\right\}$ 
is the same for both classifiers, 
since $S_1$ is treated as an i.i.d. 
sample from the population in both cases. 
\end{proof}

\addcontentsline{toc}{subsection}{C.3\quad  
   Corollary \ref{cor:main_eknn_rate_v3}}
{\bf Corollary \ref{cor:main_eknn_rate_v3}}. 
{\it
Suppose data distribution has $(\al, \be, \ga)$ 
Boltzmann margin. If  
$k = O(\tilde{n}^{\frac{\be + 2}{\be + C_\up}})$ 
and $m = O(\sqrt{\tilde{n}})$, then for 
any $\delta > 0$ and $\tilde{n} \geq 1$, 
there is 
\begin{equation*}
er(\hat{y}) - er(\hat{y}_*) \leq 
\tilde{C}_\al \delta^{-1} \exp(- 
C_\ga \tilde{n}^{\frac{\be}{\be + C_\up}}),  
\end{equation*}
with probability at least $1 - \delta$ 
over the random choice of $\tilde{S}$, 
where $\tilde{C}_\al$ is specified 
in Theorem \ref{thm:main_eknn_rate}. 
}

\begin{proof}
By Lemma \ref{lem:exer2geer} and 
Remark \ref{rem:thm_eknn_tool}, we 
have for any $t > 0$:  
\begin{equation}
{\Pr}_{\tilde{S}} 
\{ er(\hat{y}_e) - er(\hat{y}_*) > t\} 
\leq \frac{1}{t} {\Pr}_{X, \tilde{S}}\{ 
\hat{y}_e(X) \neq \hat{y}_*(X)\} 
\leq \frac{1}{t} \tilde{C}_\al \exp(- 
C_\ga \tilde{n}^{\frac{\be}{\be +C_\up}}). 
\end{equation}
Setting the right side to $\delta$ and 
solving for $t$ proves the corollary. 
\end{proof}

\addcontentsline{toc}{subsection}{C.4\quad  
   Theorem \ref{thm:main_eknn_consistency}}
{\bf Theorem \ref{thm:main_eknn_consistency}}. 
{\it Suppose data distribution has a 
$(\al, \be, \ga)$ Boltzmann margin. 
Then, $\hat{y}_e$ is strongly 
Bayes consistent 
if $k = O(\tilde{n}^{\frac{\be + 2}{\be + C_\up}})$ 
and $m = \tilde{O}(e^{\tilde{n}})$. 
}

\begin{proof}
By Lemma \ref{lem:exer2geer}, for any $\tilde{S}$, 
\begin{equation}
er(\hat{y}_e) - er(\hat{y}_*) 
\leq {\Pr}_{X}\{ \hat{y}_e(X) \neq \hat{y}_*(X) \} 
:= Q(\tilde{S}).
\end{equation}
We will prove strong consistency based 
on the following definition i.e., for any $t > 0$, 
\begin{equation}
{\Pr}_{\tilde{S}} 
\{{\lim}_{|\tilde{S}| \to \infty} 
Q(\tilde{S}) > t\} = 0.
\end{equation}
According to the Borel-Cantelli Lemma 
\ref{lem:borelcantelli}, it suffices to show  
\begin{equation}
\label{eq:thm_knn_consistency_prf_01}
\sum_{|\tilde{S}| = N_*}^\infty  
{\Pr}_{\tilde{S}}\{Q(\tilde{S}) > t \} 
< \infty,   
\end{equation}
for some integer $N_*$. 

Recall $|\tilde{S}| = m \tilde{n}$. 
Our observation is that, in order to 
verify (\ref{eq:thm_knn_consistency_prf_01}), 
it is sufficient to verify the following 
with $m = 1$:  
\begin{equation}
\label{eq:thm_knn_consistency_prf_01_ntonly}
\sum_{\tilde{n} = N_*}^\infty  
{\Pr}_{\tilde{S}}\{Q(\tilde{S}) > t \} 
< \infty,   
\end{equation}
The reason is, the sequence of 
$|\tilde{S}|$ generated with 
divergent $m$ and divergent $\tilde{n}$ 
is a subsequence of $|\tilde{S}|$ 
generated with $m = 1$ and divergent $\tilde{n}$. 
Therefore, if the latter sequence 
converges, the former sequence also 
converges. 

Fix $m = 1$. We will set $N_*$ in (\ref{eq:thm_knn_consistency_prf_01_ntonly}) 
from two perspectives. 

Let $t_s > 0$ be a variable depending 
on $\tilde{S}$, and inspect each ${\Pr}_{\tilde{S}} 
\{Q(\tilde{S}) > t_s \}$. 
Remark \ref{rem:thm_eknn_tool} implies 
that for $k = O(\tilde{n}^{\frac{\be + 2}{\be + C_\up}})$, 
$m = \tilde{O}(e^{\tilde{n}})$ (which includes 
$m = 1$) and large enough $\tilde{n}$ 
that satisfies $\tilde{n} \geq \left( 2 \ln \tilde{n} 
/C' \right)^{1+C_\up/\be}$ (which sets $N_*$):  
\begin{equation}
\label{eq:eknnstrongcon}
{\Pr}_{X, \tilde{S}}\{ \hat{y}_e(X) 
\neq \hat{y}_*(X)\} \leq \tilde{C}_\al \exp(- C_\ga 
\tilde{n}^{\frac{\be}{\be + C_\up}}). 
\end{equation}
Then, by Lemma \ref{lem:exer2geer}, we have 
\begin{equation}
{\Pr}_{\tilde{S}} 
\{Q(\tilde{S}) > t_s \} \leq \frac{1}{t_s} 
\E_{\tilde{S}} Q(\tilde{S}) 
= \frac{1}{t_s} {\Pr}_{X, \tilde{S}} 
\{ \hat{y}_e(X) \neq \hat{y}_*(X)\} 
\leq \frac{\tilde{C}_\al}{t_s} \exp(- C_\ga 
\tilde{n}^{\frac{\be}{\be + C_\up}}).
\end{equation}
Set $t_s = \exp(- \frac{C_\ga}{2} 
\tilde{n}^{\frac{\be}{\be + C_\up}})$. 
Then $t_s \to 0$ as $\tilde{n} \to \infty$. 
This means for any $t > 0$, when $
\tilde{n}$ is large enough (which  
sets $N_*$ from another perspective), 
there is $t_s < t$ and hence 
\begin{equation}
\label{eq:knnstrong_our}
{\Pr}_{\tilde{S}}\{Q({\tilde{S}}) > t \} \leq 
{\Pr}_{\tilde{S}}\{Q({\tilde{S}}) > t_s \} \leq 
\frac{\tilde{C}_\al}{t_s} \exp(- C_\ga 
\tilde{n}^{\frac{\be}{\be + C_\up}}) 
= \tilde{C}_\al \exp(- \frac{1}{2} C_\ga 
\tilde{n}^{\frac{\be}{\be + C_\up}}).
\end{equation}
Plugging the above back to 
(\ref{eq:thm_knn_consistency_prf_01_ntonly}), 
we have that for any $t > 0$, 
\begin{equation}
\sum_{\tilde{n} = N_*}^\infty  
{\Pr}_{\tilde{S}}\{Q({\tilde{S}}) > t \}\leq 
\sum_{\tilde{n} = N_*}^\infty  
{\Pr}_{\tilde{S}}\{Q({\tilde{S}}) > t_s \} 
\leq \tilde{C}_\al
\sum_{\tilde{n} = N_*}^\infty  
 \exp(- \frac{1}{2} C_\ga 
\tilde{n}^{\frac{\be}{\be + C_\up}}) < \infty.     
\end{equation}
This verifies (\ref{eq:thm_knn_consistency_prf_01_ntonly})
and further (\ref{eq:thm_knn_consistency_prf_01}) with 
any divergent $m$. Finally, note we still need 
$m = \tilde{O}(e^{\tilde{n}})$ in order to apply 
Remark \ref{rem:thm_eknn_tool} and obtain 
(\ref{eq:eknnstrongcon}). This proves the theorem.  
\end{proof}

\section{Numerical Results}

\subsection{Density Design}
\label{sec:app_numerical_density}

Recall $\X = [0, 1]$ 
and $\eta(X) = X$. 
Given constants $\al, \be, \ga > 0$, 
our goal is to design a density 
function $f_b$ on $[0,1]$ that satisfies 
$(\al, \be, \ga)$ Boltzmann margin. 
The basic idea is threefold: data 
decay near decision boundary, data 
are constant elsewhere (for better 
visualization and does not affect 
Boltzmann margin), and density 
is H\"older continuous. 

Pick a threshold $T \in (0, 0.5)$ 
and consider two cases. 

When $|X - 0.5| \leq T$, we ask data 
to decay as if under Boltzmann margin. 
This implies 
\begin{equation}
\label{eq:exp_sim_01}
\Pr \{|\eta(x) - 0.5| \leq t\} 
= \Pr \{|x - 0.5| \leq t\}
= 2 \int_{0.5}^{0.5+t} f_b(x) dx
\leq \al \exp(-\ga t^{-\be}). 
\end{equation}
For the last inequality, taking derivative 
of $t$ on both sides and setting $t' = 0.5 + t$ 
yields 
\begin{equation}
f_b(t') \leq \frac{1}{2} \al 
\be \ga e^{-\ga (t'-0.5)^{-\be}} 
(t'-0.5)^{- \be - 1}.    
\end{equation}
A mirrored result of (\ref{eq:exp_sim_01})
can be obtained by taking integral 
on $[0.5-t, 0.5]$ instead of 
$[0.5, 0.5+t]$, which yields 
\begin{equation}
f_b(t') \leq \frac{1}{2} \al 
\be \ga e^{-\ga (0.5- t')^{-\be}} 
(0.5-t')^{- \be - 1}.    
\end{equation}
Combining both results, we can set 
\begin{equation}
\label{eq:both_fb_01}
f_b(x) = \frac{1}{N_T} \cdot \frac{1}{2} \al \be 
\ga e^{-\ga |x-0.5|^{-\be}} 
|x-0.5|^{- \be - 1},    
\end{equation}
where $N_T$ is a normalizer (to be specified 
later). This completes the density design 
for $|X - 0.5| \leq T$. 

When $|x-0.5| \geq T$, we ask density 
to stay constant i.e., $f_b(x) = \frac{C_T}{N_T}$ 
for some constant $C_T$. 

It remains to specify $N_T$ and $C_T$. 
Since $f_b(x)$ is continuous 
at $|0.5-X| = T$, based on (\ref{eq:both_fb_01}) 
we have 
\begin{equation}
C_T = \frac{1}{2} \al \be 
\ga e^{-\ga T^{-\be}} T^{- \be - 1}.    
\end{equation}
Further, $\int_0^1 f_b(x) = 1$ implies 
that (through straight calculations) 
\begin{equation*}
\label{sim:NT}
N_T = \al \exp(-\ga T^{-\be}) 
+ 2 C_T (0.5 - T).
\end{equation*}
Putting all together, 
we complete the density 
design on $\X$, i.e., 
\begin{equation}
\label{eq:density_blz}
f_b(x) = \begin{cases}
\frac{\al \be \ga}{2N_T} \cdot 
\frac{e^{-\frac{\ga}{|x-0.5|^{\be}}}}{|x-0.5|^{\be + 1}}, 
& |x - 0.5| \leq T\\[.75em]  
\frac{\al \be \ga}{2N_T} \cdot 
\frac{e^{-\ga T^{-\be}}}{T^{\be + 1}}, 
& |x - 0.5| > T. 
\end{cases}    
\end{equation}

\subsection{Data Sampling}
\label{sec:app_numerical_sample}

Our goal is to sample 1000 points 
in $\X$ according to the Boltzmann-margin 
density \ref{eq:density_blz_main}, 
by applying the rejection sampling \cite{bishop2006pattern} technique. 

To be specific, recall $\X = [0, 1]$. 
Let $g(x) = 1$ be the 
density of a uniform distribution 
on $\X$, and $M > 0$ be a constant such that 
$f_b(x) \leq M g(x) = M$ for all $x \in \X$. 
We estimate $M$ from $f_b(x)$ 
based on 1000 points uniformly sampled 
in $\X$. Detail sampling and estimation 
process is described in Algorithm \ref{alg:data}. 

\begin{algorithm}[h!]
\caption{Data Generation}
\label{alg:data}
\setstretch{1}
\begin{algorithmic}[1]
\Input Data pool $\X = \emptyset$, 
pool size $N_p$, parameter $\be$.  
\State Sample a point 
$x \in [0, 1]$ uniformly at random.
\State Sample a point 
$u \in [0, M]$ uniformly at random.
\State If $u \leq f_\be(x)$, add $x$ 
to $\X$ (allow duplicate points). 
\State Repeat (1-3) until $|\X| \geq N_p$.
\For{$x \in \X$}
\State Sample a point $v \in [0,1]$ 
uniformly at random.
\State Label $x$ by 1 if $v \leq 
\eta(x)$; label it by 0 otherwise. 
\EndFor
\end{algorithmic}
\end{algorithm}

\end{document}